%% file: iclr2027_arxiv.tex
\documentclass{article}
\usepackage[T1]{fontenc}
\usepackage{iclr2027_conference,times}

\input{math_commands.tex}

\usepackage{amssymb}
\usepackage{amsthm}
\newtheorem{proposition}{Proposition}
\newtheorem{theorem}{Theorem}[section]
\newtheorem{lemma}[theorem]{Lemma}
\theoremstyle{definition}
\newtheorem{assumption}{Assumption}[section]
\theoremstyle{plain}
\usepackage{booktabs}
\usepackage{multirow}
\usepackage{graphicx}
\usepackage{placeins}
\usepackage{wrapfig}
\usepackage{subcaption}
\usepackage{microtype}
\usepackage{xcolor}
\usepackage{colortbl}
\usepackage{hhline}
\usepackage{xspace}
\usepackage{hyperref}
\usepackage{url}
\usepackage{textcomp}
\usepackage{tcolorbox}
\tcbuselibrary{breakable}
\usepackage{enumitem}

\definecolor{memoryyellow}{HTML}{FFF0C7}
\definecolor{gptheader}{RGB}{226,228,233}
\definecolor{gptbody}{RGB}{243,245,247}
\definecolor{qwenheader}{RGB}{224,215,245}
\definecolor{qwenbody}{RGB}{242,238,252}
\tcbset{
    answerprompt/.style={colback=memoryyellow!45!white,colbacktitle=memoryyellow,colframe=memoryyellow!55!black},
    judgeprompt/.style={colback=qwenbody,colbacktitle=qwenheader,colframe=qwenheader!55!black},
    auxprompt/.style={colback=gptbody,colbacktitle=gptheader,colframe=gptheader!55!black}
}
\newtcolorbox{benchmarkprompt}[2]{
    breakable,#1,title={#2},title after break={#2 (continued)},
    coltitle=black,fonttitle=\rmfamily\fontsize{8.5}{10}\selectfont\bfseries,
    fontupper=\rmfamily\fontsize{8}{9.2}\selectfont,
    boxrule=0.4pt,arc=2pt,left=5pt,right=5pt,top=4pt,bottom=4pt,
    before skip=6pt,after skip=6pt,pad at break*=1mm,
    before upper={\setlength{\parindent}{0pt}\setlength{\parskip}{0pt}\raggedright}
}

\newcommand{\method}{\textsc{MemPursuit}\xspace}

\title{Recall What Is Missing: Composing Evidence from Long-Term Agent Memory}

\author{Qi Feng \\
School of Data Science \\
The Chinese University of Hong Kong, \\
Shenzhen \\
\texttt{qifeng@link.cuhk.edu.cn}
\And
Jicong Fan\thanks{Corresponding author.} \\
School of Data Science \\
The Chinese University of Hong Kong, \\
Shenzhen \\
\texttt{fanjicong@cuhk.edu.cn}
}

\iclrfinalcopy
\begin{document}

\maketitle
\lhead{Preprint}

\begin{abstract}
Under a limited context budget, a long-term agent must recall not merely memories that are individually relevant, but evidence that is collectively sufficient for the current query. Conventional memory access retrieves, scores, and truncates memories, often spending multiple context slots on redundant facts while leaving complementary evidence behind. We cast the read phase of long-term memory as \emph{query-conditioned evidence composition}: each memory is valued by what it adds beyond the evidence already selected. We instantiate this operation with \method, a training-free mechanism that uses the residual of a sparse convex query reconstruction as a representation-space proxy for what the current memory set is missing. By reusing the query and candidate embeddings already available from retrieval, \method progressively composes nested evidence contexts without invoking a language model or cross-encoder to score each query--memory pair. The effectiveness of our method is verified on 3 popular benchmarks in comparison to strong baselines. In fixed-budget LoCoMo experiments, \method improves end-to-end accuracy by 7.2 percentage points at Top-10 and 5.5 percentage points at Top-20, with consistent gains across multiple memory frameworks, while achieving a 32.9$\times$ speedup in median reranking latency over fully batched BGE-v2-m3. Our implementation is available at \url{https://github.com/FengQi-HITSZ/MemPursuit}.
\end{abstract}

\section{Introduction}
\label{sec:introduction}

Long-lived language-model agents accumulate far more interaction history than can be placed in the prompt of every future query. External memory systems address this mismatch by extracting, organizing, and retrieving persistent records of past interactions~\citep{zhong2024memorybank,chhikara2025mem0,kang2025memoryos,xu2025amem}. Yet storage alone does not determine whether an agent remembers successfully. At answer time, a memory system must assemble a limited context from its growing store of records. We refer to this process as memory reading. Hierarchical and graph-based architectures exploit the organization of stored memories to guide retrieval~\citep{kang2025memoryos,gutierrez2024hipporag,rasmussen2025zep}. A common memory-read policy selects records according to their individual relevance to the query, using dense semantic similarity, sparse lexical matching, or hybrid retrieval~\citep{zhong2024memorybank,pan2025secom,fang2026lightmem}, optionally followed by a relevance reranker~\citep{li2026memreranker}.

For the dominant pointwise memory-read, the final context consists of the $K$ highest individual scores, implicitly treating the usefulness of a memory as independent of the other selected memories~\citep{sun2026barrag}. In long histories, independently relevant memories often repeat the same event, preference, or conversational fact~\citep{hu2026xmemory}. A multi-hop query can consequently receive several paraphrases of one hop while missing a less similar fact required to complete the evidence chain. Figure~\ref{fig:teaser} shows this pattern: the standard read repeats memories about Deborah's father while omitting her friend Karlie, whereas \method selects complementary records about her father, mother, and Karlie to complete the answer. Increasing $K$ only partially addresses this failure: it consumes more context, exposes the reader to more distractors, and weakens the purpose of a selective memory layer~\citep{ji2026mragent}.  The key decision is not just which memories match the query, but which memory adds what the current context still lacks.

We formulate this read-time decision as \emph{query-conditioned evidence composition}: assembling a compact context by valuing each memory according to what it adds to the evidence already selected. Graph-based systems already support cross-memory retrieval through explicit relational structures~\citep{gutierrez2024hipporag,rasmussen2025zep}. However, building and maintaining these structures requires additional operations, with graph-wide consolidation identified as a substantial source of overhead as memory grows~\citep{zhou2026agentnative}. 

\begin{figure}[t]
    \centering
    \vspace{-15pt}
    \includegraphics[width=0.95\linewidth]{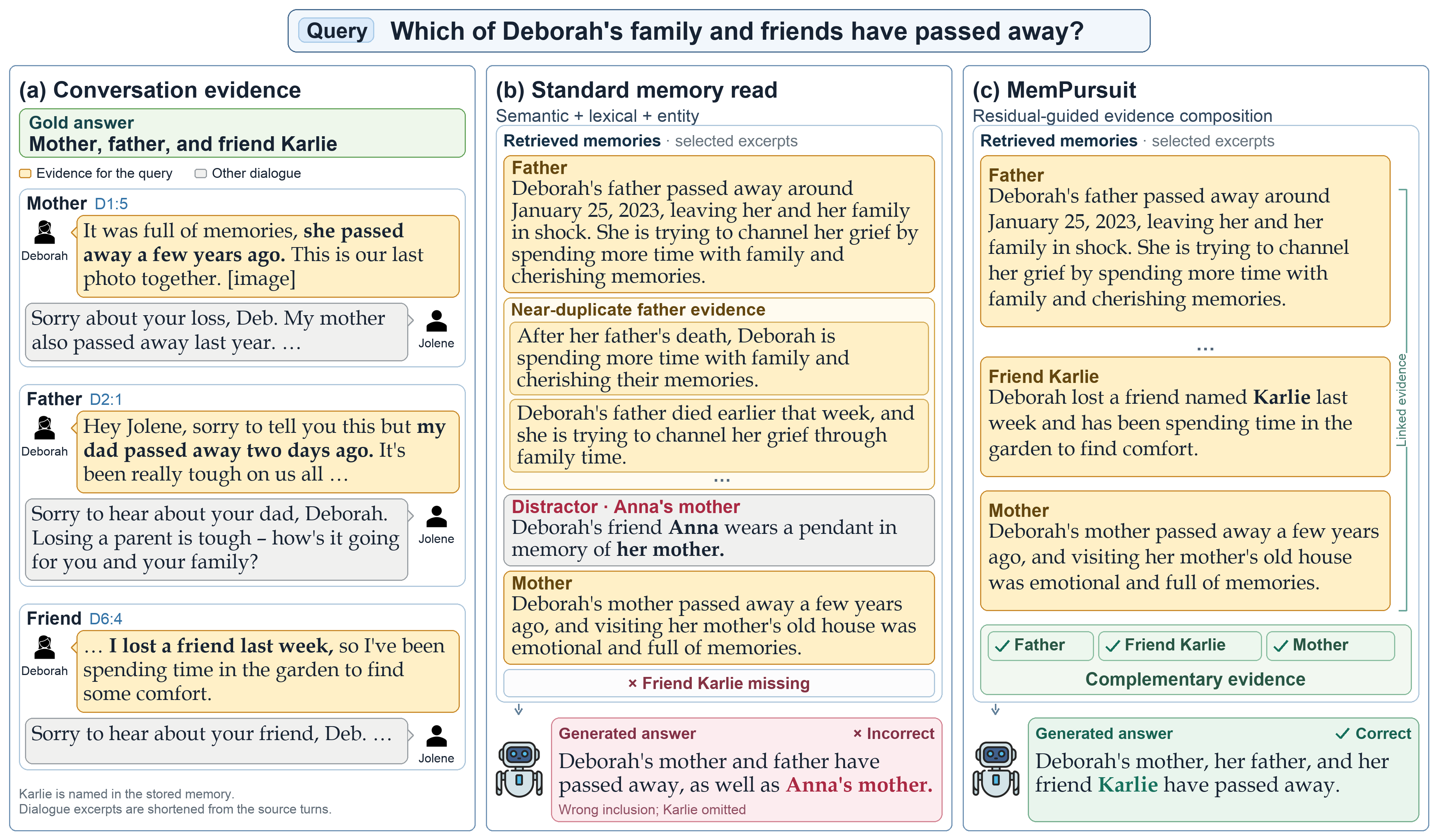}
    \caption{A memory-reading example from LoCoMo. The query asks which of Deborah's family and friends have passed away. The standard read returns repeated records about her father and a distractor about Anna's mother, while omitting her friend Karlie. \method selects complementary records about her father, mother, and Karlie, enabling the complete answer shown. Dialogue and memory excerpts are shortened for presentation. Additional case studies appear in Appendix~\ref{app:missing-evidence-utility}.}
    \label{fig:teaser}
    \vspace{-20pt}
\end{figure}

We introduce \method, a training-free reader that implements this principle through sparse convex query reconstruction. Its central idea is to represent the query using the selected memory embeddings and use the remaining residual to guide further recall. Starting with the strongest semantic match, \method alternates between selecting a memory aligned with the residual and refitting the weights of all selected memories. The residual discounts query directions already represented in the context, directing subsequent selections toward complementary content. A single pursuit path yields nested evidence sets across context budgets. By operating directly on existing query and memory embeddings, \method makes these conditional choices through matrix and vector operations, without additional language-model or cross-encoder scoring.

In controlled comparisons on LoCoMo, with the memory store, candidate pool, and downstream evaluation fixed, \method improves Top-10 accuracy by 7.2 percentage points, including a 12.8-point gain on multi-hop questions. It improves gold-session coverage and reaches 85.1\% accuracy with 20 memories, outperforming hybrid retrieval with 50. At Top-20, it comes within 0.5 points of BGE-large with 32.7$\times$ lower median selection latency. These benefits transfer across four memory frameworks, and its end-to-end pipeline achieves the highest aggregate scores among the evaluated systems on LoCoMo, LongMemEval-S, and BEAM with both answerers. These results establish residual-guided composition as an effective, efficient, and reusable operation for long-term agent memory. Our contributions are threefold:
\begin{itemize}[leftmargin=10pt]
    \item \textbf{Residual-guided evidence composition.} We introduce \method, a training-free memory reader that turns query reconstruction into a set-conditioned selection rule. Residual scoring and fully corrective simplex refitting adapt each choice to the evidence already selected.
    \item \textbf{More complete evidence under limited budgets.} Controlled comparisons show higher answer accuracy with fewer memories and improve complete gold-session coverage from 77.1\% to 84.5\% at Top-10. Reconstruction analysis connects these evidence gains to the geometric selection signal.
    \item \textbf{Reusable and efficient memory reading.} Experiments span three benchmarks and two answerers, with accuracy gains in all eight framework--budget configurations in cross-framework comparisons. Reranker comparisons demonstrate competitive answer quality at a fraction of the selection cost.
\end{itemize}

\section{Preliminaries}
\label{sec:problem}

\subsection{Writing memories, reading evidence}
\label{sec:writing-reading}

An agent memory system performs two complementary operations. \emph{Writing} constructs and updates a persistent memory store $\mathcal{M}=\{m_i\}_{i=1}^{N}$ from past interactions. \emph{Reading} assembles stored information into the answer model's context for a query $q$. The writer determines what the agent retains. The reader chooses which memories to use together for the current answer.

For each query, the underlying memory system supplies a candidate pool $\mathcal{C}_P(q)\subseteq\mathcal{M}$ of $P$ memories. A read policy $\pi$ selects an evidence set $S_K(q)=\pi(q,\mathcal{C}_P(q),K)\subseteq\mathcal{C}_P(q)$ with $|S_K(q)|\le K$. The selected memory texts are passed to the answer model together with $q$. The budget counts memory entries in the form supplied by the backbone. We study this final selection step, which determines how the available evidence is assembled into the answer context.

The backbone handles memory writing, store organization, and candidate generation. To isolate the effect of memory selection, our controlled comparisons hold the store, candidate pool, budget, answerer, and evaluation protocol fixed.

\subsection{Memory reading as budgeted set selection}
\label{sec:conditional-utility}

Let $U(q,S)$ denote the utility of memory set $S$ for answering query $q$ with a fixed answer model. Memory reading seeks a set that maximizes this utility within the available budget:
\begin{equation}
    S^{\star}\in\arg\max_{\substack{S\subseteq\mathcal{C}_P(q)\\|S|\le K}}U(q,S).
    \label{eq:read-policy}
\end{equation}
Pointwise selection approximates set utility by adding individual relevance scores, $U_{\mathrm{pw}}(q,S)=\sum_{m\in S}g(q,m)$. For a fixed set size $K$, maximizing this sum gives the Top-$K$ memories. Each memory contributes the same score regardless of what else is selected. The additive objective therefore cannot capture how overlapping evidence reduces a memory's contribution or how complementary facts become useful together. Since answer quality is not known at selection time, we need a computable objective that evaluates memories jointly. We use query reconstruction to measure how well a combination of memories represents the query and to guide the choice of that combination.

\section{MemPursuit}
\label{sec:method}

\begin{figure}[t]
    \centering
\includegraphics[width=\linewidth]{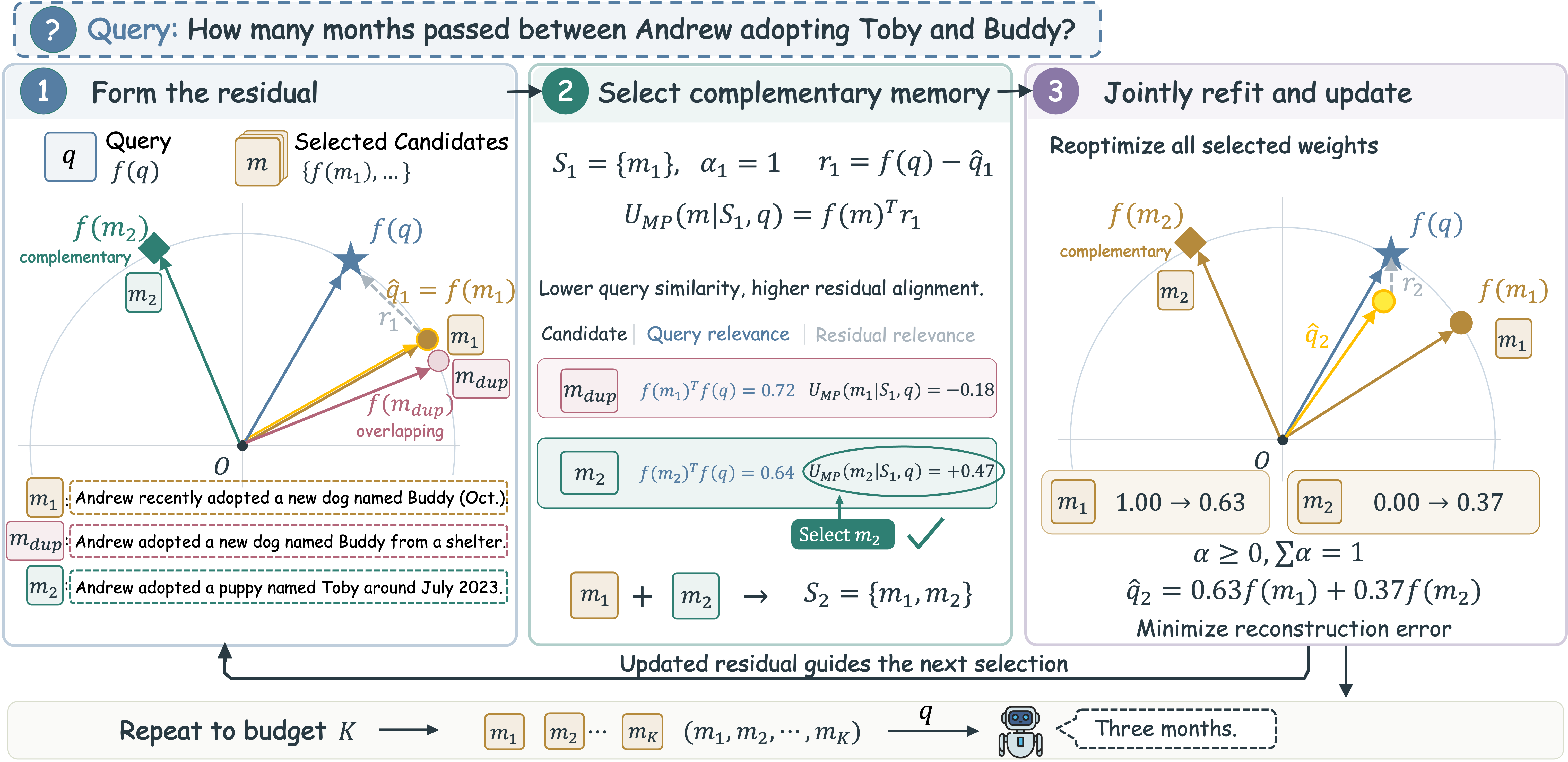}
    \caption{Overview of \method, illustrated with a query requiring two adoption dates. (1) Form the residual between the query embedding and the current memory reconstruction. (2) Select the memory with the highest residual alignment: the complementary record about Toby is preferred over another record about Buddy, despite its lower query similarity. (3) Jointly refit the selected memories' nonnegative, sum-to-one weights and update the residual. Repeat until the memory budget is reached or the adaptive stopping condition is met. Scores and weights are illustrative.}
    \label{fig:method-overview}
    \vspace{-10pt}
\end{figure}

\subsection{Evidence composition through sparse query reconstruction}
\label{sec:query-reconstruction}

To instantiate the set utility in Eq.~\ref{eq:read-policy}, we evaluate how closely a weighted combination of memory embeddings can represent the query. This couples each memory's contribution to the rest of the set: similar embeddings offer similar reconstruction directions, while complementary embeddings can bring the joint reconstruction closer to the query.

Let $\mathcal{C}=\mathcal{C}_P(q)$, and let a fixed encoder $f:\mathcal{X}\rightarrow\mathbb{R}^d$ map query and memory texts to unit-norm embeddings. We formulate memory read as sparse query reconstruction over the full candidate pool:
\begin{equation}
\begin{aligned}
    \min_{\boldsymbol{\alpha}\in\mathbb{R}^{P}}\quad &
    \frac{1}{2}\left\lVert f(q)-\sum_{m\in\mathcal{C}}\alpha_m f(m)\right\rVert_2^2+
    \frac{\lambda}{2}\lVert\boldsymbol{\alpha}\rVert_2^2\\
    \text{s.t.}\quad &\boldsymbol{\alpha}\succeq0,\quad \sum_{m\in\mathcal{C}}\alpha_m=1,\quad \lVert\boldsymbol{\alpha}\rVert_0\le K.
\end{aligned}
    \label{eq:objective}
\end{equation}
The simplex constraints make the reconstruction a convex combination of selected embeddings. It prevents the fit from exploiting cancellation between positive and negative coefficients while fixing the total weight to prevent unrestricted rescaling. The cardinality constraint enforces the memory budget, and the ridge term with $\lambda>0$ stabilizes the fit when memories overlap.

Eq.~\ref{eq:objective} defines a combinatorial optimization problem, but fitting the weights becomes convex once the memory set is fixed. We therefore use a fully corrective greedy pursuit that expands the set one memory at a time and jointly refits its weights~\citep{tropp2007signal,locatello2017unified}. Each refit involves only the selected memories, and its residual guides the next expansion.

\subsection{Fully corrective residual pursuit}
\label{sec:residual-selection}
\label{sec:joint-refit}

The next memory should contribute to the part of the query that the current combination has not yet represented. For a selected set $S$ with fitted weights $\alpha_m(q,S)$, this direction is captured by the reconstruction residual:
\begin{equation}
    \hat{\mathbf{q}}(q,S)=\sum_{m\in S}\alpha_m(q,S)f(m),
    \qquad \mathbf{r}(q,S)=f(q)-\hat{\mathbf{q}}(q,S).
    \label{eq:residual}
\end{equation}
We evaluate an unselected memory by its alignment with this remaining direction:
\begin{equation}
    U_{\mathrm{MP}}(m\mid q,S)=f(m)^{\top}\mathbf{r}(q,S).
    \label{eq:memory-utility}
\end{equation}
Because an unselected memory has zero weight, this score equals the negative derivative of the objective in Eq.~\ref{eq:objective} with respect to its coefficient. Residual alignment thus provides a first-order criterion for expanding the selected set, without requiring a separate refit for every candidate. At step $t$, we select
\begin{equation}
    m_t^{\star}=\arg\max_{m\in\mathcal{C}\setminus S_{t-1}}f(m)^{\top}\mathbf{r}(q,S_{t-1}),
    \qquad S_t=S_{t-1}\cup\{m_t^{\star}\}.
    \label{eq:selection}
\end{equation}
We initialize $S_0=\varnothing$ and $\hat{\mathbf{q}}(q,S_0)=\mathbf{0}$, so the first selection is the strongest query cosine match. Subsequent selections depend on the residual left by the jointly fitted set.

Adding a memory can change how earlier selections should contribute to the reconstruction. Keeping their weights fixed would carry the previous allocation into the next selection, even when the expanded set admits a better fit. We therefore reoptimize all selected weights:
\begin{equation}
    \boldsymbol{\alpha}_t=\underset{\substack{\boldsymbol{\alpha}\succeq0\\\sum_{m\in S_t}\alpha_m=1}}{\arg\min}
    \frac{1}{2}\left\lVert f(q)-\sum_{m\in S_t}\alpha_m f(m)\right\rVert_2^2+
    \frac{\lambda}{2}\lVert\boldsymbol{\alpha}\rVert_2^2.
    \label{eq:refit}
\end{equation}
The fitted weights update the reconstruction and residual in Eq.~\ref{eq:residual}. Selection identifies a direction for expanding the representation, while refitting redistributes contributions across the expanded set. Together, these operations make every subsequent choice depend on the current joint reconstruction.

\subsection{Context construction and adaptive stopping}
\label{sec:context-construction}

The greedy pursuit produces an ordered sequence of memories with nested selected sets $S_1\subset S_2\subset\cdots$. For a fixed budget $K$, \method performs $K$ selection--refit steps and uses the texts of the first $K$ selected memories as the answer context. A single pursuit path also supplies contexts for smaller budgets through its prefixes. Composition weights guide the search, while all selected entries remain in the context even if their fitted weights later become zero.

This incremental construction also allows context size to adapt to the query. Instead of always using the full budget, \method can terminate when further selections make little progress in query reconstruction. After each refit, we measure the squared reconstruction error and its relative improvement:
\begin{equation}
    L_t=\lVert\mathbf{r}(q,S_t)\rVert_2^2,
    \qquad
    \Delta_t=\frac{L_{t-1}-L_t}{\max(L_{t-1},\epsilon_{\mathrm{num}})}.
    \label{eq:adaptive-stopping}
\end{equation}
Here $\epsilon_{\mathrm{num}}>0$ stabilizes the denominator near zero. Adaptive selection stops when the relative improvement remains at or below a tolerance $\tau$, or when the budget is reached. This yields a context of $k_q\le K$ memories and avoids further context expansion and computation once reconstruction gains become small. 

\subsection{From reconstruction to evidence coverage}
\label{sec:coverage-guarantee}

The following proposition connects residual-guided selection to evidence coverage. It first bounds the scores of memories close to the selected set's convex hull. Under a fact-cluster model that groups memories supplying the same fact, it then gives conditions for selecting all required facts before repeating any of them.

\begin{proposition}[Redundancy avoidance and evidence coverage]
\label{prop:redundancy-coverage}
Assume unit-norm embeddings and exact refits with a fixed $\lambda\ge0$. For nonempty $S$, let $H_S=\operatorname{conv}\{f(s):s\in S\}$. For $\delta\ge0$, any candidate $m$ with $\operatorname{dist}(f(m),H_S)\le\delta$ satisfies
\begin{equation}
U_{\mathrm{MP}}(m\mid q,S)
\le \max_{s\in S}f(s)^{\top}\mathbf{r}(q,S)+\delta\lVert\mathbf{r}(q,S)\rVert_2.
\label{eq:redundancy-bound}
\end{equation}
Thus, whenever an unselected memory exceeds this bound, the next selection lies outside the $\delta$-neighborhood of $H_S$.
Under the fact-cluster separation and coverage-margin conditions in Assumption~\ref{ass:fact-clusters}, \method, run for at least $h$ steps, selects exactly one representative of each of the $h$ required facts and no distractors in its first $h$ selections. Consequently, every returned prefix with budget $K\ge h$ covers all required facts.
\end{proposition}

The result explains how residual updates favor uncovered facts before spending additional slots on already represented evidence. Appendix~\ref{app:theory} gives the full assumptions and proof, together with a computable optimality certificate for the reconstruction objective.

\subsection{Complexity analysis}
\label{sec:complexity}

Table~\ref{tab:complexity-comparison} compares the cost of final memory selection for \method and a pointwise LLM scorer. Let $P$ be the candidate count, $d$ the embedding dimension, and $K\le P$ the context budget. For the LLM scorer, $L$ denotes the number of Transformer layers, $W$ the hidden width, and $n$ the maximum number of tokens per scoring sequence, including both prompt and continuation tokens. We analyze selection from available candidates and embeddings, excluding query encoding, candidate retrieval, and answer generation. Appendix~\ref{app:computational-complexity} details the complexity calculations.

\begin{table}[!htbp]
\centering
\caption{Time and space complexity of final memory selection. Numerical scales use $K=50$, $P=200$, $d=1536$, $L=32$, $W=2560$, and $n=128$.
}
\label{tab:complexity-comparison}
\small
\setlength{\tabcolsep}{4pt}
\renewcommand{\arraystretch}{1.25}
\begin{tabular}{@{}lcccc@{}}
\toprule
& \multicolumn{2}{c}{\method} & \multicolumn{2}{c}{LLM scorer} \\
\cmidrule(lr){2-3}\cmidrule(lr){4-5}
Resource & Asymptotic cost & Instantiated scale & Asymptotic cost & Instantiated scale \\
\midrule
Time & $O(P^2d+KP^2+K^4)$ & $6.969\!\times\!10^7$ & $O(PL(nW^2+n^2W))$ & $5.637\!\times\!10^{12}$ \\
Space & $O(Pd+P^2)$ & $3.472\!\times\!10^5$ & $O(LW^2+LnW+n^2)$ & $2.202\!\times\!10^8$ \\
\bottomrule
\end{tabular}
\vspace{0.3em}

\end{table}

\method replaces repeated token-level neural evaluation with matrix operations on existing embeddings and small selected-set refits. This removes dependence on scorer depth and width while making each selection responsive to the evolving evidence context. Section~\ref{sec:efficiency-results} evaluates the resulting advantage in measured selection latency.

\section{Experiments}
\label{sec:experiments}
We compare \method with nine baselines across three benchmarks using two answerers, and evaluate its integration into four memory frameworks. Further analyses examine evidence coverage, performance across question types and memory budgets, and quality--efficiency trade-offs with neural rerankers. 
\subsection{Experimental setup}

\textbf{Benchmarks \& Baselines.}
We evaluate our proposed method end-to-end on 3 well-known benchmarks: LoCoMo~\citep{maharana2024locomo}, LongMemEval~\citep{wu2024longmemeval}, and BEAM~\citep{tavakoli2026beam}, with different average context lengths of 16.6k, 115k, and 500k tokens respectively. The system comparison includes full context, Pure RAG, A-MEM~\citep{xu2025amem}, LightMem~\citep{fang2026lightmem}, SimpleMem~\citep{liu2026simplemem}, MemoryOS~\citep{kang2025memoryos}, Mem0~\citep{chhikara2025mem0}, MemoryBank~\citep{zhong2024memorybank}, and MemOS~\citep{li2025memos}.

\textbf{Evaluation protocol.}
In the end-to-end experiments, we evaluate \method and baselines with two answer LLMs, GPT-4o-mini~\citep{openai2024gpt4omini} and Qwen3-30B-A3B-Instruct-2507~\citep{qwen3technicalreport}. All systems use the same answer prompt, judge model, and judge prompt. Full-system comparisons retain each baseline's native memory construction and retrieval settings. The \method pipeline in Table~\ref{tab:systems} uses Mem0-v3 OSS as its memory backbone with a memory read budget of $K=50$. We report binary-judge QA accuracy for LoCoMo and LongMemEval-S and the official grader score for BEAM.

To evaluate whether \method improves memory reading across different frameworks, we compare LightMem, MemoryBank, SimpleMem, and Mem0 with and without \method on LoCoMo at $K\in\{10,20\}$. For each framework, the two variants share the same memory store, candidate pool, and selection budget, with GPT-4o-mini serving as both answerer and judge.
We fix the simplex-refit ridge coefficient to $\lambda=10^{-10}$ across all experiments. Full implementation details are provided in Appendix~\ref{app:implementation-details}.

\subsection{End-to-end performance and cross-framework comparisons}
\label{sec:transfer-results}

\label{sec:system-results}

\paragraph{Strong performance across long-term memory benchmarks.}
\method with Mem0-v3 backbone achieves the highest aggregate score among the evaluated systems on all three benchmarks with both answerers as demonstrated in Table~\ref{tab:systems}. With GPT-4o-mini, it reaches 87.08\% on LoCoMo, 71.20\% on LongMemEval-S, and 43.82 on BEAM. The overall lead persists with Qwen3-30B-A3B, demonstrating effective memory reading across downstream models and interaction histories. 

\begin{table}[!t]
\caption{End-to-end results under two answerer backbones. LoCoMo and LongMemEval-S report binary-judge accuracy (\%); BEAM reports its official mean score scaled to 0--100. 
Bold and underline denote the best and second-best result within each answerer panel. 
``--'' denotes an unavailable run due to limited context length.}
\label{tab:systems}
\centering
\scriptsize
\arrayrulecolor{black}
\setlength{\tabcolsep}{2.2pt}
\resizebox{\linewidth}{!}{%
\begin{tabular}{lrrrrrrrrrrrrr}
\toprule
& \multicolumn{5}{c}{LoCoMo} & \multicolumn{7}{c}{LongMemEval-S} & \multicolumn{1}{c}{BEAM} \\
\cmidrule(lr){2-6}\cmidrule(lr){7-13}\cmidrule(lr){14-14}
Method & MH & Temp. & Open & SH & Avg. & KU & MS & SSA & SSP & SSU & TR & Avg. & Avg. \\
\midrule
\rowcolor{gptheader}
\multicolumn{14}{c}{\rule{0pt}{2.2ex}\raisebox{-0.25ex}{\includegraphics[height=1.8ex]{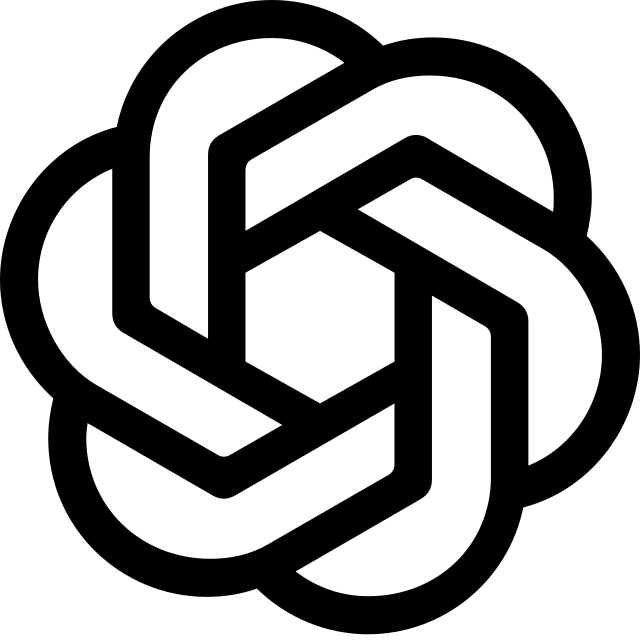}}\hspace{0.5em}\textbf{GPT-4o-mini}\rule[-0.7ex]{0pt}{0pt}} \\
\midrule
\rowcolor{gptbody}
Full context & 68.79 & 50.16 & 56.25 & 86.56 & 73.83 & -- & -- & -- & -- & -- & -- & -- & -- \\
\rowcolor{gptbody}
RAG & 82.62 & 58.57 & \underline{64.58} & 85.85 & 78.25 & 53.85 & 45.86 & \textbf{94.64} & 46.67 & 87.14 & 41.35 & 57.20 & 38.23 \\
\rowcolor{gptbody}
A-MEM & 59.22 & 64.17 & 54.17 & 69.80 & 65.71 & 69.23 & 55.64 & \textbf{94.64} & 40.00 & 84.29 & 63.16 & 67.20 & 38.78 \\
\rowcolor{gptbody}
LightMem & 79.08 & 73.83 & 62.50 & 86.56 & 81.04 & \textbf{80.77} & 72.93 & 12.50 & 46.67 & \underline{88.57} & \textbf{77.44} & 69.20 & \underline{41.42} \\
\rowcolor{gptbody}
SimpleMem & 79.79 & 53.27 & \underline{64.58} & 86.21 & 76.82 & 53.85 & 38.35 & 33.93 & 40.00 & 71.43 & 41.35 & 45.80 & 36.37 \\
\rowcolor{gptbody}
MemoryOS & 53.90 & 46.42 & 42.71 & 63.26 & 56.75 & 70.51 & 57.89 & 75.00 & 43.33 & 74.29 & 55.64 & 62.60 & 31.05 \\
\rowcolor{gptbody}
Mem0 & 64.54 & 53.89 & 51.04 & 72.41 & 65.78 & 78.21 & \textbf{78.20} & 16.07 & 46.67 & \textbf{92.86} & \underline{74.44} & \underline{70.40} & 36.05 \\
\rowcolor{gptbody}
MemoryBank & 32.27 & 23.68 & 44.79 & 47.80 & 39.74 & 39.74 & 20.30 & \underline{85.71} & 36.67 & 70.00 & 26.32 & 40.20 & 34.67 \\
\rowcolor{gptbody}
MemOS & \underline{82.98} & \underline{77.57} & 63.54 & \underline{90.01} & \underline{84.48} & 71.79 & 64.66 & 35.71 & \underline{50.00} & 82.86 & 69.17 & 65.40 & 35.77 \\ \hline
\rowcolor{gptbody}
\textbf{\method} & \textbf{84.40} & \textbf{80.06} & \textbf{65.62} & \textbf{93.10} & \textbf{87.08} & \underline{79.49} & \underline{73.68} & 57.14 & \textbf{60.00} & \textbf{92.86} & 60.90 & \textbf{71.20} & \textbf{43.82} \\
\midrule
\rowcolor{qwenheader}
\multicolumn{14}{c}{\rule{0pt}{2.2ex}\raisebox{-0.25ex}{\includegraphics[height=1.8ex]{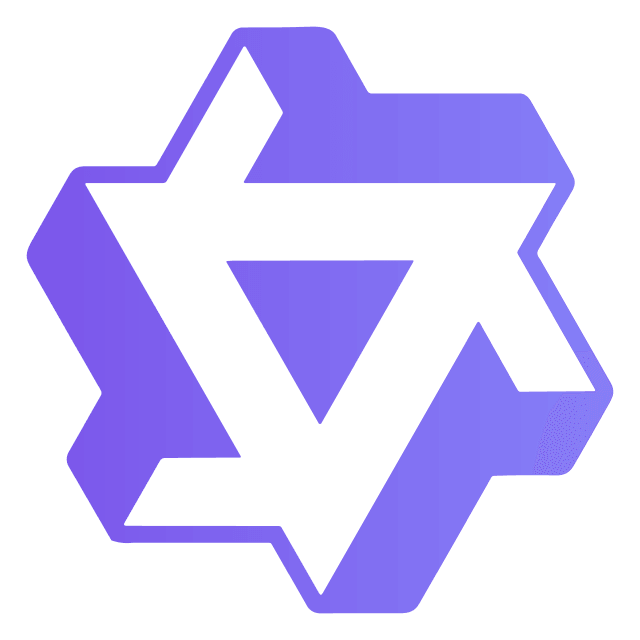}}\hspace{0.5em}\textbf{Qwen3-30B-A3B-Instruct-2507}\rule[-0.7ex]{0pt}{0pt}} \\
\midrule
\rowcolor{qwenbody}
Full context & 69.86 & 51.71 & 57.29 & 87.40 & 74.87 & -- & -- & -- & -- & -- & -- & -- & -- \\
\rowcolor{qwenbody}
RAG & \underline{81.21} & 72.59 & 66.67 & \underline{90.73} & \underline{83.70} & 67.95 & 44.36 & \underline{92.86} & 53.33 & 90.00 & 46.62 & 61.00 & 35.22 \\
\rowcolor{qwenbody}
A-MEM & 58.16 & 61.06 & 61.46 & 69.32 & 65.06 & 73.08 & 54.14 & \textbf{96.43} & \underline{83.33} & \underline{94.29} & 63.91 & 71.80 & 34.27 \\
\rowcolor{qwenbody}
LightMem & 78.37 & 72.27 & 63.54 & 84.66 & 79.61 & \underline{88.46} & \textbf{72.18} & 16.07 & \textbf{86.67} & 91.43 & 72.93 & \underline{72.20} & \underline{38.88} \\
\rowcolor{qwenbody}
SimpleMem & 78.01 & 54.83 & \textbf{70.83} & 84.54 & 76.30 & 56.41 & 39.10 & 44.64 & 70.00 & 82.86 & 42.86 & 51.40 & 37.39 \\
\rowcolor{qwenbody}
MemoryOS & 56.74 & 47.35 & 52.08 & 64.21 & 58.57 & 69.23 & 57.14 & 73.21 & 80.00 & 82.86 & 66.92 & 68.40 & 31.12 \\
\rowcolor{qwenbody}
Mem0 & 64.18 & 53.27 & 55.21 & 68.97 & 63.96 & 83.33 & 69.92 & 17.86 & 73.33 & \textbf{97.14} & \textbf{75.94} & 71.80 & 37.28 \\
\rowcolor{qwenbody}
MemoryBank & 36.17 & 22.74 & 51.04 & 49.11 & 41.36 & 42.31 & 20.30 & 89.29 & 66.67 & 75.71 & 31.58 & 45.00 & 32.77 \\
\rowcolor{qwenbody}
MemOS & 76.24 & \underline{78.19} & \underline{68.75} & 85.85 & 81.43 & 78.21 & 66.17 & 41.07 & 76.67 & 88.57 & \underline{74.44} & 71.20 & 36.08 \\ \hline
\rowcolor{qwenbody}
\textbf{\method} & \textbf{81.91} & \textbf{80.37} & 63.54 & \textbf{90.96} & \textbf{85.39} & \textbf{91.03} & \underline{71.43} & 62.50 & 76.67 & \textbf{97.14} & 62.41 & \textbf{75.00} & \textbf{43.66} \\
\bottomrule
\end{tabular}}
\vspace{0.3em}

\parbox{0.98\linewidth}{\scriptsize \textbf{Abbreviations.} MH = multi-hop, Temp. = temporal, Open = open-domain, SH = single-hop. KU = knowledge update, MS = multi-session, SSA = single-session assistant, SSP = single-session preference, SSU = single-session user, TR = temporal reasoning.}
\vspace{-5pt}
\end{table}

\begin{table}[htbp]
\caption{Cross-framework comparisons on LoCoMo at $K=10$. Each system is evaluated without (w/o) and with (w/) \method at the same selection budget. Entries report accuracy (\%); parentheses in the w/ rows give the change in percentage points from the paired w/o result. For each system, the higher accuracy in each metric is bolded. Yellow highlights the w/ results. All configurations evaluate 1,540 questions using GPT-4o-mini as answerer and judge. Full results at $K\in\{10,20\}$ appear in Table~\ref{tab:cross-system-read-full} (Appendix~\ref{app:cross-framework-results}).}
\label{tab:cross-system-read}
\centering
\small
\arrayrulecolor{black}
\setlength{\tabcolsep}{2.5pt}
\renewcommand{\arraystretch}{1.25}
\newcommand{\memgain}[2]{#1\,{\scriptsize($#2$)}}
\newcommand{\membase}[1]{#1\phantom{\,{\scriptsize($+0.00$)}}}
\begin{tabular}{@{}l|clllll}
\hline
System & {\scriptsize\method} & Multi-hop & Temporal & Open-domain & Single-hop & Overall \\
\hline
\multirow{2}{*}{LightMem} & w/o & \membase{62.77} & \membase{66.98} & \membase{56.25} & \membase{77.17} & \membase{71.10} \\
 & w/ & \cellcolor{memoryyellow}\memgain{\textbf{69.15}}{+6.38} & \cellcolor{memoryyellow}\memgain{\textbf{68.22}}{+1.25} & \cellcolor{memoryyellow}\memgain{\textbf{57.29}}{+1.04} & \cellcolor{memoryyellow}\memgain{\textbf{79.19}}{+2.02} & \cellcolor{memoryyellow}\memgain{\textbf{73.70}}{+2.60} \\
\hhline{~|------}
\multirow{2}{*}{MemoryBank} & w/o & \membase{53.55} & \membase{\textbf{43.30}} & \membase{51.04} & \membase{65.87} & \membase{57.99} \\
 & w/ & \cellcolor{memoryyellow}\memgain{\textbf{56.74}}{+3.19} & \cellcolor{memoryyellow}\memgain{39.25}{-4.05} & \cellcolor{memoryyellow}\memgain{\textbf{56.25}}{+5.21} & \cellcolor{memoryyellow}\memgain{\textbf{70.27}}{+4.40} & \cellcolor{memoryyellow}\memgain{\textbf{60.45}}{+2.47} \\
\hhline{~|------}
\multirow{2}{*}{SimpleMem} & w/o & \membase{62.41} & \membase{48.60} & \membase{56.25} & \membase{74.91} & \membase{65.97} \\
 & w/ & \cellcolor{memoryyellow}\memgain{\textbf{70.21}}{+7.80} & \cellcolor{memoryyellow}\memgain{\textbf{49.22}}{+0.62} & \cellcolor{memoryyellow}\memgain{\textbf{60.42}}{+4.17} & \cellcolor{memoryyellow}\memgain{\textbf{82.28}}{+7.37} & \cellcolor{memoryyellow}\memgain{\textbf{71.82}}{+5.84} \\
\hhline{~|------}
\multirow{2}{*}{Mem0} & w/o & \membase{45.04} & \membase{42.37} & \membase{43.75} & \membase{57.43} & \membase{51.17} \\
 & w/ & \cellcolor{memoryyellow}\memgain{\textbf{55.32}}{+10.28} & \cellcolor{memoryyellow}\memgain{\textbf{49.84}}{+7.48} & \cellcolor{memoryyellow}\memgain{\textbf{48.96}}{+5.21} & \cellcolor{memoryyellow}\memgain{\textbf{67.54}}{+10.11} & \cellcolor{memoryyellow}\memgain{\textbf{60.45}}{+9.29} \\
\hline
\end{tabular}
\end{table}

\textbf{Reusable reading improvement when context is scarce.} The cross-framework comparisons in Table~\ref{tab:cross-system-read} show higher overall accuracy in all eight framework--budget configurations with \method. At Top-10, gains range from 2.47\% to 9.29\%; improvements also hold at Top-20. The largest and most consistent gains occur on multi-hop and single-hop questions, with accuracy improvements ranging from 0.71\% to 10.28\% and from 0.59\% to 10.11\% respectively across the eight configurations. These results show that \method improves memory reading across frameworks with different memory construction strategies, enabling each framework to make more effective use of a limited context budget. The following analysis traces these gains to the evidence entering the selected context.

\subsection{Does MemPursuit recall what is missing?}
\label{sec:mechanism-results}

To examine whether MemPursuit recalls evidence missed by relevance-based selection, we track how it changes the selected context and its coverage of annotated gold sessions. We compare MemPursuit with the original hybrid order on LoCoMo at Top-10 and Top-20, using the same pool of 200 candidates per question and keeping the rest of the Mem0-v3 pipeline fixed in Figure~\ref{fig:selection-evidence-coverage}.  
\method replaces 70.8\% of the baseline's selected memories at Top-10 and 64.5\% at Top-20. These changes increase macro gold-session recall from 82.9\% to 90.3\% at Top-10 and from 91.3\% to 94.8\% at Top-20. Complete coverage rises from 77.1\% to 84.5\% and from 86.2\% to 90.3\%, respectively. Newly recalled memories contribute more gold-session coverage than is lost through displaced memories. Residual-guided selection therefore brings previously omitted evidence sources into the context, allowing the same memory budget to cover more of the query-required sources.

\begin{figure}[!htbp]
\centering
\captionsetup[subfigure]{justification=centering}
\begin{subfigure}[t]{0.300\textwidth}
    \centering
    \includegraphics[width=\linewidth]{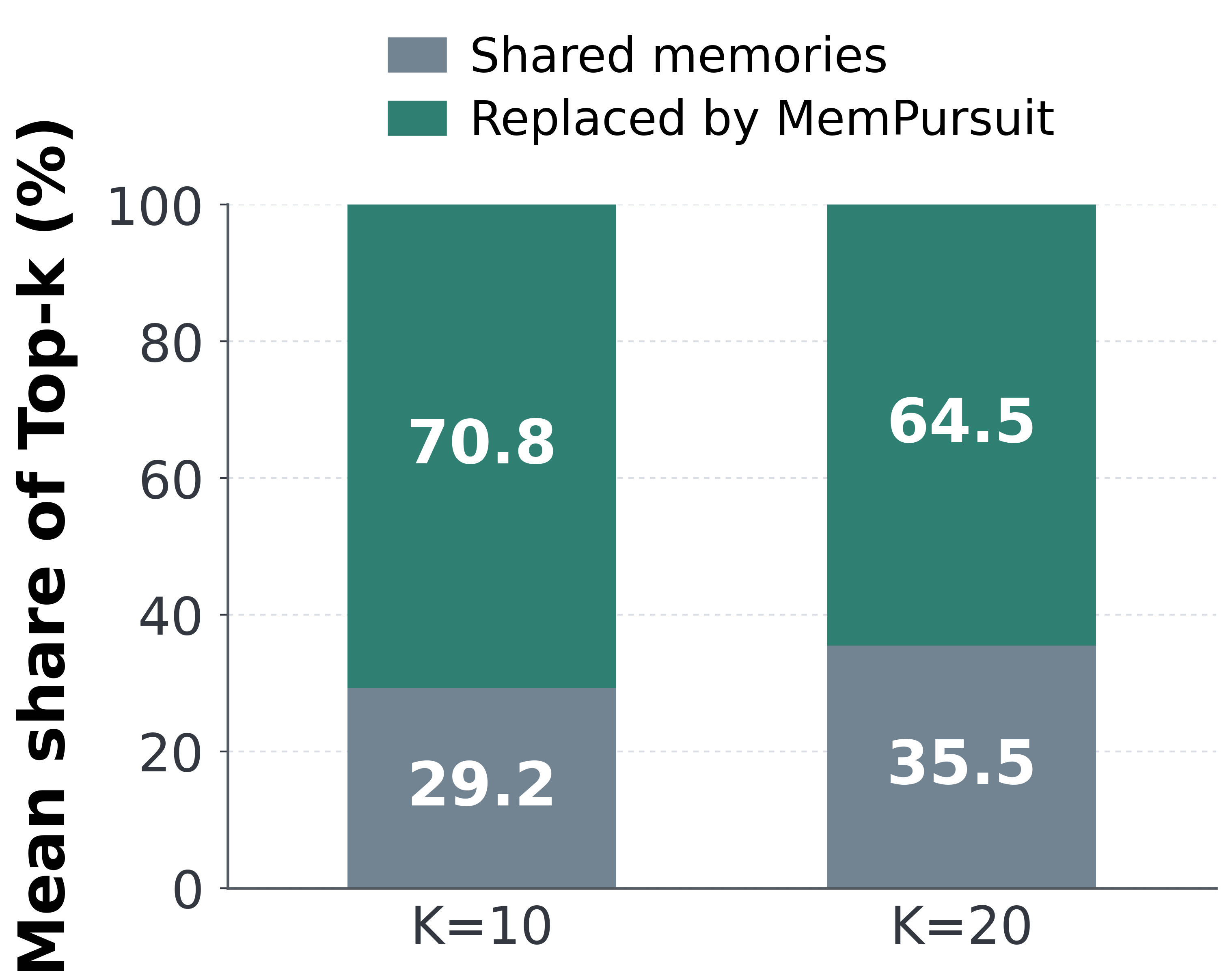}
    \caption{Memory selection overlap}
    \label{fig:memory-selection-overlap}
\end{subfigure}
\hfill
\begin{subfigure}[t]{0.355\textwidth}
    \centering
    \includegraphics[width=\linewidth]{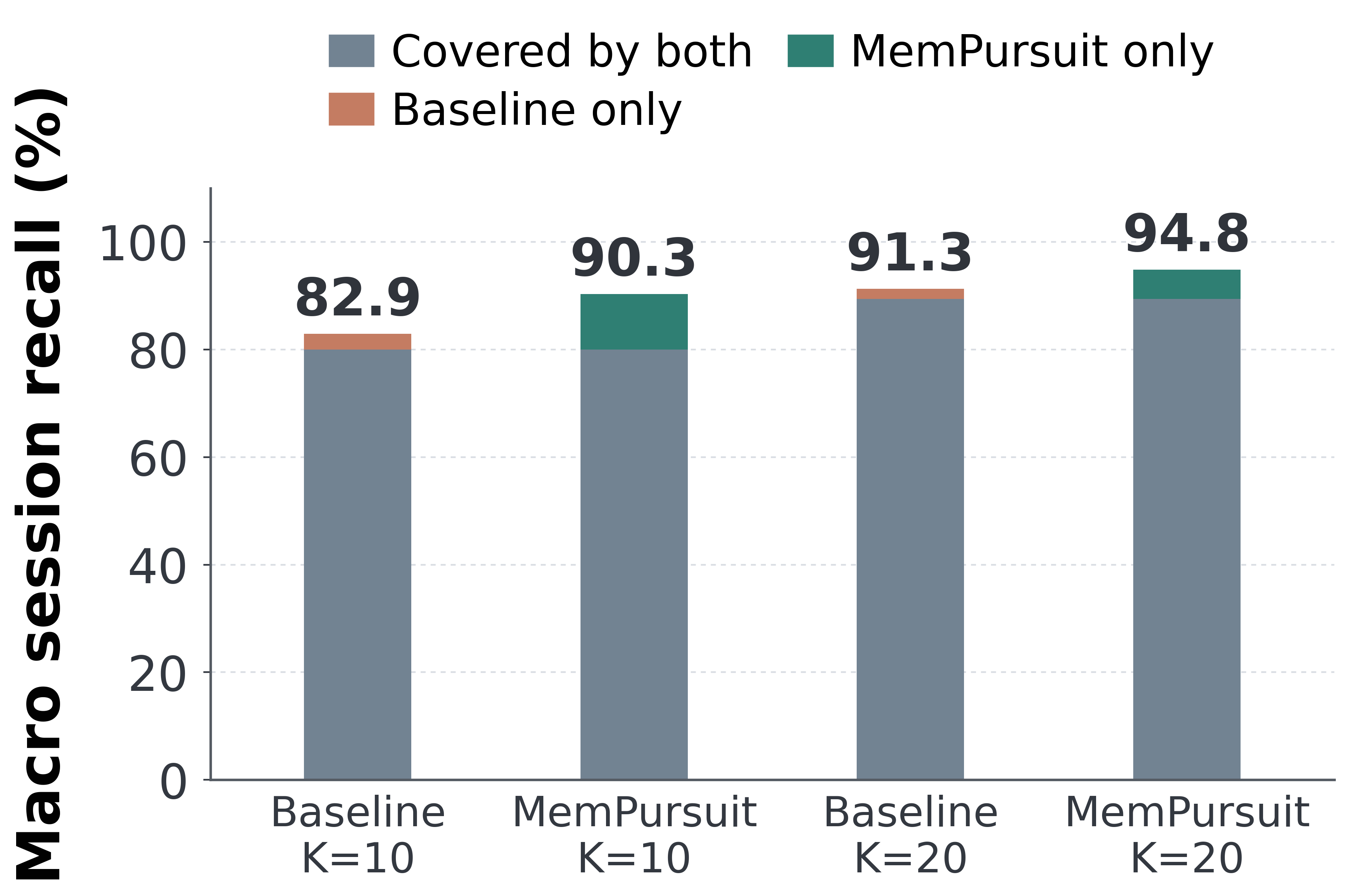}
    \caption{Gold evidence-session coverage}
    \label{fig:gold-session-coverage}
\end{subfigure}
\hfill
\begin{subfigure}[t]{0.308\textwidth}
    \centering
    \includegraphics[width=\linewidth]{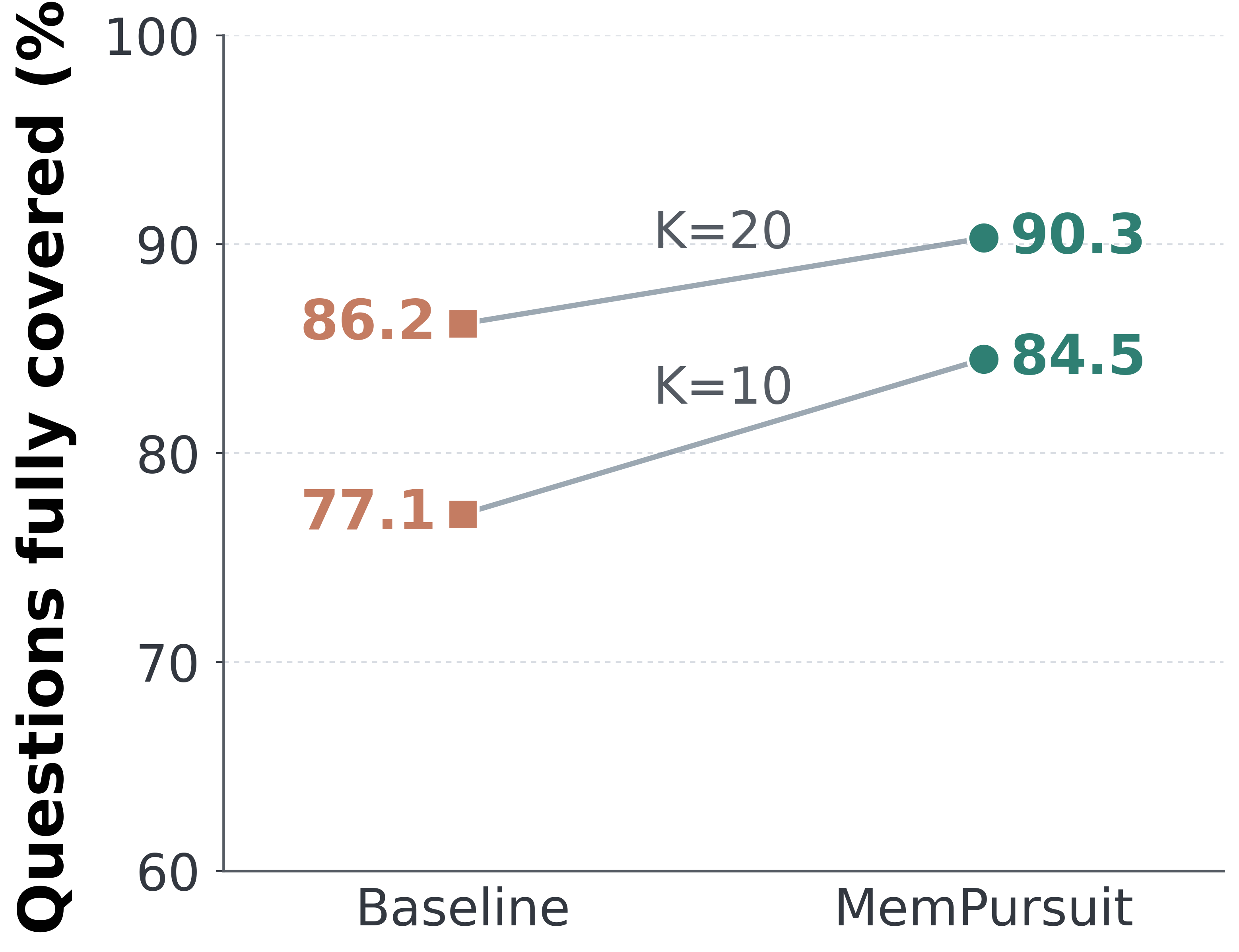}
    \caption{Complete evidence-session coverage}
    \label{fig:complete-session-coverage}
\end{subfigure}
\vspace{-10pt}
\caption{Memory selection and evidence-session coverage on LoCoMo at Top-10 and Top-20. 
Panels (b) and (c) evaluate 1,536 gold-annotated questions. \method changes the selected context while increasing both average and complete evidence-session coverage.}
\label{fig:selection-evidence-coverage}
\end{figure}

\subsection{Performance across question types and context budgets}
\label{sec:read-policy-results}
\label{sec:budget-results}

We next examine how the changes in selected evidence translate into answer accuracy across question types and memory budgets. Using the same candidate pools and evaluation setup as Section 4.3, we compare \method with the original hybrid order.
\begin{figure*}[htb]
\centering
\begin{subfigure}[t]{0.47\textwidth}
    \centering
    \includegraphics[width=\linewidth]{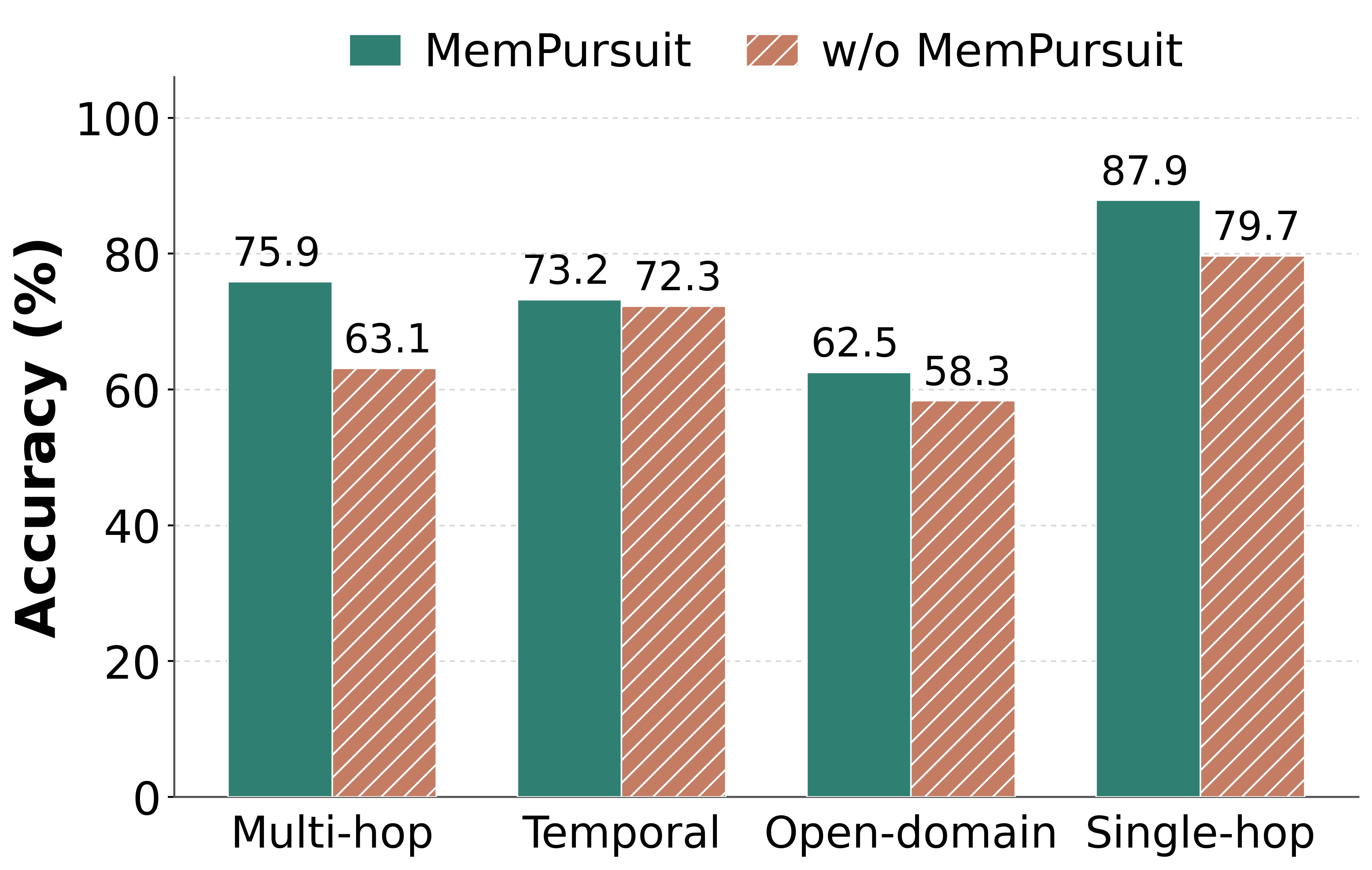}
    \caption{Question-type accuracy at Top-10.}
    \label{fig:locomo-k10-question-types}
\end{subfigure}
\hfill
\begin{subfigure}[t]{0.47\textwidth}
    \centering
    \includegraphics[width=\linewidth]{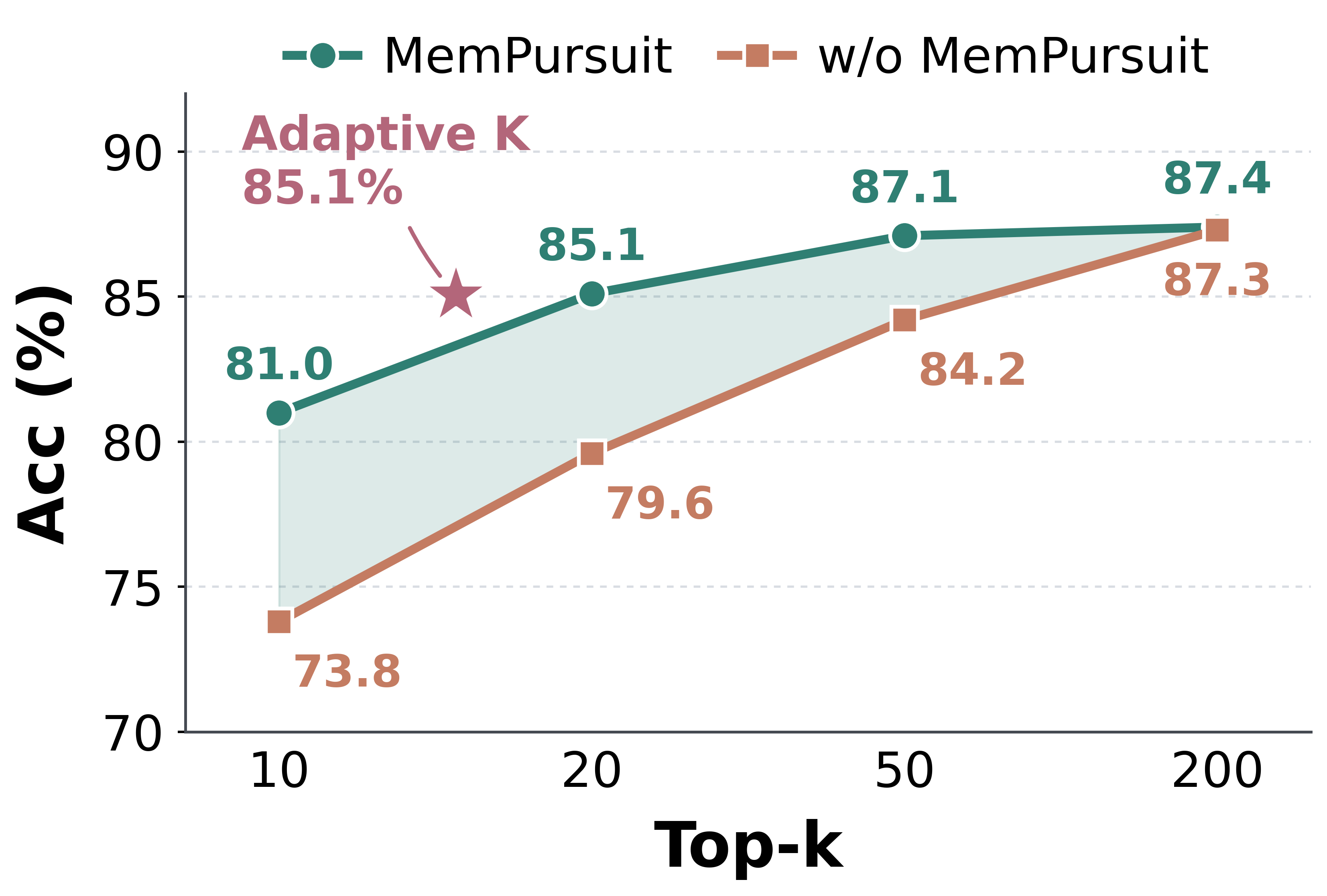}
    \caption{Overall accuracy across Top-$K$ context budgets.}
    \label{fig:budget-curve}
\end{subfigure}
\vspace{-5pt}
\caption{LoCoMo performance with GPT-4o-mini as the answerer. 
The star in panel (b) marks adaptive context selection using the stopping rule in Section~\ref{sec:context-construction} with $\tau=10^{-5}$, achieving 85.1\% accuracy with an average of 15.65 selected memories. The 0.1\% accuracy gap at Top-200 arises from differences in memory ordering.
}
\vspace{-10pt}
\label{fig:locomo-quality-views}
\end{figure*}

\textbf{The strongest gain occurs on multi-hop questions.}
In Figure~\ref{fig:locomo-k10-question-types}, Multi-hop accuracy rises from 63.1\% to 75.9\%, 
while single-hop accuracy increases from 79.7\% to 87.9\% . Temporal and open-domain questions also improve at this budget. The particularly large multi-hop gain fits the evidence-composition mechanism: allocating context across complementary sources is valuable when several facts must be available together. Single-hop gains show that more effective context allocation also benefits focused factual recall. 

\textbf{Better evidence enters the context earlier.} In Figure~\ref{fig:budget-curve}, \method reaches 85.1\% accuracy with 20 memories, exceeding hybrid retrieval with 50 memories. It can be observed that the advantage is largest at the smallest tested budget and narrows as the two policies include more of the candidate pool. Moreover, \method already approaches its Top-200 accuracy with only 50 memories. These patterns show that \method assembles useful evidence within a smaller memory budget, whereas relevance-based ordering requires a larger context to achieve comparable answer accuracy.

\subsection{Quality--efficiency trade-offs with neural rerankers}
\label{sec:efficiency-results}

Finally, we compare \method with BGE-base, BGE-large, and BGE-v2-m3 cross-encoder rerankers~\citep{xiao2023cpack,chen2024bgem3}, as well as Qwen3.5-0.8B, -2B, and -4B pointwise LLM scorers~\citep{qwenteam2026qwen35} in Figure~\ref{fig:reranker-accuracy-latency}. Appendix~\ref{app:efficiency} provides the complete latency statistics and batch-size comparisons.
Relative to hybrid retrieval, \method captures approximately 73\% of the accuracy gain delivered by BGE-v2-m3, the most accurate reranker in this comparison, while requiring only 3.2\% of its median reranking latency. It also exceeds BGE-base and Qwen3.5-0.8B in answer accuracy and comes within half a percentage point of BGE-large and Qwen3.5-2B. 
Even when BGE-v2-m3 processes all 200 candidates in a single batch, \method achieves a 32.9$\times$ speedup in median reranking latency. \method therefore combines strong answer quality with low online selection latency, even against fully batched neural scoring. 

\begin{wrapfigure}{r}{0.45\textwidth}
\vspace{-40pt}
\captionsetup{font=footnotesize}
\centering
\includegraphics[width=\linewidth]{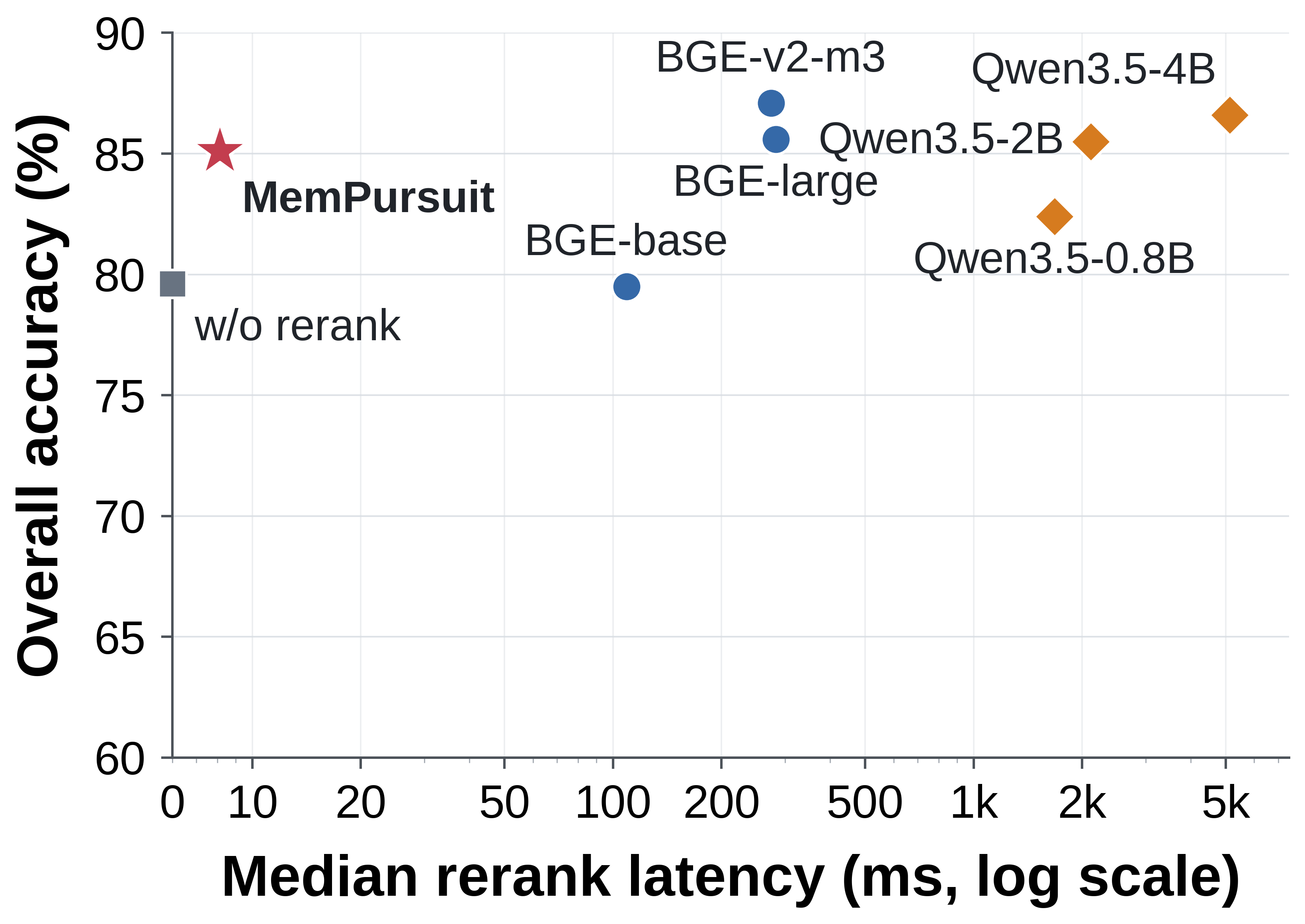}
\vspace{-15pt}
\caption{LoCoMo accuracy versus median reranking latency at Top-20 with $P=200$. Neural scorers use batch size 128. Latency includes final reranking only.}
\label{fig:reranker-accuracy-latency}
\vspace{-25pt}
\end{wrapfigure}
\subsection{Additional experimental results}
\label{sec:more-results}

Additional experiments test whether gold evidence reconstructs the query more closely than retrieved alternatives (Appendix~\ref{app:residual-analysis}) and whether residual scoring favors missing evidence given an already selected evidence memory (Appendix~\ref{app:missing-evidence-utility}). Case studies are in the same appendix.  We compare residual-guided selection with MMR in Appendix~\ref{app:mmr-comparison}. Appendix~\ref{app:efficiency} reports latency comparisons with neural rerankers across batch sizes and examines how \method's selection latency scales with candidate count $P$ and budget $K$.

\section{Related Work}
\label{sec:related-work}

\textbf{Long-term memory for language-model agents.}
Agent memory systems improve retention and access through persistent memory management~\citep{zhong2024memorybank,chhikara2025mem0}, segmentation and compression~\citep{pan2025secom,fang2026lightmem,liu2026simplemem}, and linked or graph-structured organization~\citep{xu2025amem,gutierrez2024hipporag,rasmussen2025zep}. MemoryOS and MemOS develop broader memory-management architectures~\citep{kang2025memoryos,li2025memos}, while Hindsight distinguishes factual evidence, experiences, and beliefs~\citep{latimer2025hindsight}. Memory-R1 learns memory management and utilization through reinforcement learning~\citep{yan2025memoryr1}. These systems shape what agents retain and how they access it. \method focuses on composing the final answer context from a supplied candidate pool, providing a reusable read operation across memory architectures.

\textbf{From relevance ranking to evidence composition.}
Dense retrieval such as DPR~\citep{karpukhin2020dpr} and pointwise rerankers such as MemReranker~\citep{li2026memreranker} estimate individual relevance, whereas RankGPT and ListT5 use listwise comparisons to produce rankings~\citep{sun2023rankgpt,yoon2024listt5}. Set-dependent selection has a longer history in MMR~\citep{carbonell1998mmr}, xQuAD~\citep{santos2010xquad}, submodular optimization~\citep{lin2011submodular}, and determinantal point processes~\citep{kulesza2012dpp}, which promote novelty, aspect coverage, or diversity. More recent approaches target evidence utility directly~\citep{deng2025influence,singh2026qubo}. For instance, SetR uses LLM reasoning to select passages satisfying query requirements~\citep{lee2025setr}. \method instead evaluates the selected set through query reconstruction. Jointly refitted weights determine the residual guiding each addition, making selection sensitive to what the current set already represents without a separately trained utility model or LLM-based selection.

\textbf{Dynamic retrieval and residual pursuit.}
MemR$^3$ and GAM interleave retrieval with reflection or research~\citep{du2025memr3,yan2025gam}, while MRAgent uses LLM reasoning to prune memory-graph paths~\citep{ji2026mragent}. xMemory combines hierarchical organization with uncertainty-guided expansion~\citep{hu2026xmemory}. These approaches make retrieval responsive to accumulated evidence. \method implements this adaptation geometrically within the candidate pool. Drawing on orthogonal matching pursuit~\citep{tropp2007signal}, Frank--Wolfe optimization~\citep{jaggi2013revisiting}, and generalized matching pursuit~\citep{locatello2017unified}, it alternates residual-guided selection with fully corrective simplex refitting. This connects sparse approximation to agent-memory composition, yielding nested contexts through inexpensive operations on existing embeddings.

\section{Conclusion}
\label{sec:conclusion}
\vspace{-5pt}
A long-term memory system must compose a limited answer context from complementary evidence. MemPursuit implements this operation through simplex-constrained residual pursuit: selected memories jointly reconstruct the query, and the remaining residual guides subsequent recall. The method is training-free, reuses existing embeddings, and produces nested contexts without additional neural scoring. Experiments show consistent accuracy gains across memory frameworks, while comparisons with neural rerankers demonstrate competitive answer quality at a fraction of their selection cost. However, MemPursuit relies on the structure of the semantic embedding space and cannot explicitly reason about contextual relationships between memories and the query.

\bibliography{iclr2027_conference}
\bibliographystyle{iclr2027_conference}

\clearpage
\appendix
\section{Implementation Details}
\label{app:implementation-details}

\subsection{MemPursuit algorithm and implementation}
\label{app:algorithm}

\begin{figure}[h]
\centering
\fbox{\begin{minipage}{0.93\linewidth}
\small
\textbf{Input:} query text $q$; candidate memories $\mathcal{C}=\mathcal{C}_P(q)$; query and memory embeddings $f(q),\{f(m):m\in\mathcal{C}\}$ from the fixed normalized encoder $f$; memory budget $K$; ridge $\lambda$.\par
\textbf{Output:} memories in residual-guided entry order.\par
\begin{enumerate}
    \item Index $\mathcal{C}=\{m_1,\ldots,m_P\}$; compute $G_{ij}=f(m_i)^{\top}f(m_j)$ and $b_i(q)=f(m_i)^{\top}f(q)$; initialize $S_0\leftarrow\varnothing$.
    \item For $t=1,\ldots,K$:
    \begin{enumerate}
        \item Compute $U_{\mathrm{MP}}(m_i\mid q,S_{t-1})\leftarrow b_i(q)-\sum_{m_j\in S_{t-1}}G_{ij}\alpha_{m_j}(q,S_{t-1})$ for all candidates (use $U_{\mathrm{MP}}(m_i\mid q,S_0)=b_i(q)$) and mask selected memories.
        \item Select $m_t^{\star}\leftarrow\arg\max_{m\in\mathcal{C}\setminus S_{t-1}}U_{\mathrm{MP}}(m\mid q,S_{t-1})$ and set $S_t\leftarrow S_{t-1}\cup\{m_t^{\star}\}$.
        \item Solve Eq.~\ref{eq:refit} for $\boldsymbol{\alpha}_t$ on $S_t$, warm-started from the previous weights and a zero weight for $m_t^{\star}$.
        \item Record residual norm, marginal objective reduction, and composition weights.
    \end{enumerate}
    \item Return $(m_1^{\star},\ldots,m_{K}^{\star})$; its first $k$ memories form $S_k(q)$ for any $k\le K$.
\end{enumerate}
\end{minipage}}
\caption{Fixed-budget \method procedure. Memories are returned in entry order, without sorting by final coefficient magnitude.}
\end{figure}

\subsubsection{Numerical implementation}
\label{app:algorithm-implementation}

\paragraph{Precomputed similarities.}
Index the candidate memories as $\mathcal{C}=\{m_1,\ldots,m_P\}$ and precompute the memory--memory Gram matrix $G_{ij}=f(m_i)^{\top}f(m_j)$ and query--memory similarities $b_i(q)=f(m_i)^{\top}f(q)$. The refit vector $\boldsymbol{\alpha}_t$ contains the weights $\alpha_m(q,S_t)$ for $m\in S_t$. Using Eq.~\ref{eq:memory-utility}, the conditional selection score for memory $m_i$ is
\begin{equation}
\begin{aligned}
    U_{\mathrm{MP}}(m_i\mid q,S_{t-1})
    &=b_i(q)-\sum_{m_j\in S_{t-1}}G_{ij}\alpha_{m_j}(q,S_{t-1})\\
    &=f(m_i)^{\top}\mathbf{r}(q,S_{t-1}),
\end{aligned}
\end{equation}
allowing all candidates to be evaluated against the updated reconstruction using the precomputed Gram matrix and query similarities. These quantities are reused throughout the pursuit path.

\paragraph{Warm starting the refit.}
For $t>1$, each refit starts from the previous solution with a zero coefficient appended for the new memory. This initialization preserves simplex feasibility and reuses the preceding computation. All selected coefficients are then jointly optimized according to Eq.~\ref{eq:refit}.

\paragraph{Selection and output semantics.}
With a fixed budget, \method runs for $K$ selections and returns $(m_1^{\star},\ldots,m_{K}^{\star})$ in greedy entry order. For any $k\le K$, the first $k$ entries form the evidence set $S_k(q)=\{m_1^{\star},\ldots,m_k^{\star}\}$. The query argument in $S_t(q)$ is omitted in the main derivation for readability. Composition weights guide subsequent selection; a memory remains in the path even if its weight later becomes zero. The returned path is not reordered by final coefficient magnitude.

\paragraph{Adaptive stopping.}
Adaptive stopping applies Eq.~\ref{eq:adaptive-stopping} after each refit, with $K$ as the maximum number of selected memories. The returned context includes all entries selected up to termination. For the adaptive-budget experiment in Section~\ref{sec:budget-results}, we set the numerical floor $\epsilon_{\mathrm{num}}=10^{-12}$, the relative-improvement threshold $\tau=10^{-5}$, and the patience $T_{\mathrm{pat}}=2$. Thus selection stops after two consecutive iterations with relative reconstruction-error improvement at most $10^{-5}$. The stopping loss is the squared residual norm $L_t$ in Eq.~\ref{eq:adaptive-stopping}, excluding the ridge penalty.

\subsection{Computational complexity}
\label{app:computational-complexity}

Let $P$ be the candidate count, $d$ the embedding dimension, and $K\le P$ the context budget. We bound the cost of up to $K$ pursuit steps. We analyze selection from available candidates and embeddings, excluding query encoding, candidate retrieval, and answer generation.

\paragraph{MemPursuit.}
The full-Gram implementation constructs the Gram matrix in $O(P^2d)$ time. At step $t$, candidate scoring costs $O(Pt)$, and a dense full-candidate gradient refresh costs $O(P^2)$. For the $t$-variable simplex QP, an initial factorization costs $O(t^3)$; updating it across $O(t)$ working-set iterations at $O(t^2)$ per update gives $O(t^3)$ per refit. With one refit per step and at most $K$ steps, the total time is $O(P^2d+PK^2+KP^2+K^4)$, which simplifies to $O(P^2d+KP^2+K^4)$ because $K\le P$. The candidate embeddings, full Gram matrix, and QP factors occupy $O(Pd+P^2+K^2)=O(Pd+P^2)$ working memory. A single pursuit path supplies all smaller context budgets. Adaptive stopping can reduce the number of steps and the associated selection and refitting costs.

\paragraph{LLM-based scoring complexity.}
For a dense Transformer with $L$ layers and hidden width $W$, let $n$ bound the total number of prompt and continuation tokens in a scoring sequence. Linear projections and feed-forward transformations cost $O(nW^2)$ per layer, while attention costs $O(n^2W)$~\citep{vaswani2017attention}. Scoring $P$ query--memory pairs therefore costs
\begin{equation}
T_{\mathrm{LLM}}=O\!\left(PL(nW^2+n^2W)\right).
\end{equation}
For single-sequence inference, model parameters require $O(LW^2)$ space, the key--value cache requires $O(LnW)$, and materialized attention requires $O(n^2)$ working space. The resulting space complexity is
\begin{equation}
S_{\mathrm{LLM}}=O(LW^2+LnW+n^2).
\end{equation}
These are the time and space estimates reported in Table~\ref{tab:complexity-comparison}; the space estimate includes scorer parameters and inference states for one sequence.

\subsection{Memory backbones and baseline configurations}
\label{app:implementation}

\paragraph{Mem0-v3 memory construction and retrieval.}
The Mem0-v3 implementation processes timestamped conversations incrementally. For each incoming message batch, GPT-4o-mini at temperature $0.1$ extracts new, self-contained memories using the current messages, the 10 most recent messages, and 10 relevant existing memories retrieved by dense similarity. Facts can be extracted from both user and assistant messages, with observation timestamps resolving relative time expressions. Extraction appends new memories without overwriting earlier records. Exact duplicates are filtered against the retrieved memories and within the current batch. The resulting memories are stored in Qdrant with \texttt{text-embedding-3-small} embeddings ($d=1536$), BM25 sparse vectors derived from lemmatized text, and temporal and source metadata. Extracted entities form a separate entity-to-memory index.

At query time, dense and BM25 retrieval produce a union of candidates, which are ranked by combining dense similarity, calibrated BM25 scores, and an entity-matching boost. Entity matching adjusts scores within this union without adding candidates. The top $P=200$ memories form the candidate pool supplied to \method. After \method selects the requested prefix of its pursuit path, the selected memories are sorted by timestamp for presentation to the answer model. This step preserves the selected set while presenting events and updates in chronological order.

\paragraph{System-comparison retrieval budgets.}
We follow the baselines' native retrieval configurations, with an adjusted budget for MemoryBank. A-MEM and MemoryOS use $K=10$. A-MEM additionally includes linked memories. MemoryOS supplements retrieved memories with short-term memory, user and assistant knowledge, and a user profile. LightMem uses $K=60$ across LoCoMo, LongMemEval-S, and BEAM. SimpleMem includes all stored memories, while Mem0 uses $K=60$. For MemoryBank, we increase its default budget from $K=2$ to $K=50$ to match MemPursuit's nominal retrieval budget. MemOS retrieves the top 20 memories separately for each of the two mirrored speakers, yielding 40 memory entries in the answer context.

\paragraph{Cross-framework comparisons.}
Paired variants share the stored memories, candidate pool, and downstream processing. We reuse each framework's query and memory embeddings, apply $\ell_2$ normalization, and run \method with early stopping disabled for up to 200 selections, limited by the available candidates. Selections at $K=10$ and $K=20$ are prefixes of the same greedy selection sequence. Selection follows pursuit order rather than the fitted coefficients, without blending in the original retrieval scores. The Mem0-v3 hybrid pool is described above. The four frameworks in Tables~\ref{tab:cross-system-read} and~\ref{tab:cross-system-read-full} are integrated as follows:
\begin{itemize}
    \item \textbf{LightMem.} Both variants reuse the saved global Top-60 retrieval pool. The baseline retains the native retrieval order, while \method selects from the same pool.
    \item \textbf{SimpleMem.} We retain the complete deduplicated union of semantic, keyword, and structured retrieval results, using the embedding vectors from LanceDB. Both variants truncate this same pool at $K$, following either the original hybrid order or the pursuit order.
    \item \textbf{MemoryBank.} We retrieve the Top-200 seeds by native L2 distance, then select $K$ seeds by either L2 rank or \method. Both variants subsequently apply the same within-date neighbor expansion, merging, chronological ordering, and resident summaries. Thus, $K$ counts seeds, and the final number of context chunks can differ from $K$. These $K\in\{10,20\}$ comparisons are separate from the $K=50$ system-comparison baseline.
    \item \textbf{Mem0 Paper OSS.} We merge the two mirrored speaker stores' Top-30 results into a global pool of at most 60 memories. The baseline orders this union by saved native cosine scores, while \method runs once over the combined pool before truncation at $K$.
\end{itemize}

\begingroup
\subsection{Benchmark prompt templates}
\label{app:prompts}

We reproduce the LoCoMo prompt templates below and specify the prompt sources for LongMemEval and BEAM.

\subsubsection{LoCoMo}
\label{app:prompts-locomo}

The answer and judge template comes from the implementation of Mem0~\citep{chhikara2025mem0}. Its placeholders are \texttt{\{question\}}, \texttt{\{reference\_date\}}, and \texttt{\{memories\}}. The builder defaults the reference date to \texttt{2023}, takes up to 200 supplied memories, and presents them oldest first with human-readable dates. An optional user profile is prepended to the memories field. No retrieval scores or ranks are shown.

\begin{benchmarkprompt}{answerprompt}{LoCoMo --- Answer generation}
\noindent You are answering a question using retrieved memories from past conversations. Follow these reasoning steps IN ORDER.\par
\vspace{0.25\baselineskip}
\noindent \#\# Step 1: SCAN ALL MEMORIES\par
\noindent Read EVERY memory below from first to last. For each one that contains information relevant to the question, note it. Do NOT stop after finding the first relevant memory \textemdash{} important details are often scattered across many memories, including ones far down the list. Give equal weight to ALL memories regardless of position \textemdash{} a memory near the end is just as likely to contain the answer as one near the beginning. In these memories, \textquotedbl{}User\textquotedbl{} refers to the main person whose memories these are.\par
\vspace{0.25\baselineskip}
\noindent \#\# Step 2: ENTITY VERIFICATION\par
\noindent Confirm each relevant memory is about the correct person/entity. If the question asks \textquotedbl{}What does Person A like?\textquotedbl{} and a memory says \textquotedbl{}Person B likes X\textquotedbl{}, do NOT use that memory to answer about Person A. In two-person conversations, both speakers\textquotesingle{} actions are relevant \textemdash{} if the question asks about person A and a memory attributes an action to person B (the other speaker), that information is still valid evidence from their shared conversations, but always check the attribution is correct.\par
\vspace{0.25\baselineskip}
\noindent \#\# Step 3: COMBINE AND CROSS-REFERENCE\par
\noindent - COMBINE facts from multiple memories about the same topic. If one memory says \textquotedbl{}won first place\textquotedbl{} and another says \textquotedbl{}performed a piece titled X,\textquotedbl{} those describe the same event \textemdash{} connect them.\par
\noindent - For listing/counting questions, extract EVERY distinct item from ALL memories. A single memory may contain multiple items. Think about what CATEGORIES of answers the question could have, then re-scan specifically for each category.\par
\noindent - For counting questions (\textquotedbl{}how many times\textquotedbl{}, \textquotedbl{}how many X\textquotedbl{}), enumerate each distinct instance explicitly with its date or context BEFORE giving a final count. Do not estimate \textemdash{} list them out, then count the list.\par
\noindent - DECOMPOSE complex sentences: \textquotedbl{}an immersive X with Y, enjoys Z\textquotedbl{} contains multiple distinct facts. Each could be the answer.\par
\noindent - Connect related facts across memories: if one says \textquotedbl{}nearby lake\textquotedbl{} and another says \textquotedbl{}Lake Tahoe is great for kayaking\textquotedbl{}, the nearby lake IS Lake Tahoe. If one says \textquotedbl{}bought X in Paris\textquotedbl{}, infer the country is France.\par
\vspace{0.25\baselineskip}
\noindent \#\# Step 4: SELECT THE BEST ANSWER\par
\noindent - Do NOT assume the highest-ranked memory is correct. Multiple memories may describe different events for the same topic. Compare each candidate\textquotesingle{}s relevance to the SPECIFIC question, not its retrieval score. A lower-ranked memory that directly answers the question beats a higher-ranked one that is only tangentially related.\par
\noindent - ALWAYS choose the MOST SPECIFIC detail available. A proper name, title, or number beats a generic description. Rate each candidate as HIGH specificity (name, title, number, specific activity) or LOW (generic description), and prefer HIGH.\par
\noindent - Report what someone actually DID, not what was offered or available to them. \textquotedbl{}Has not tried X yet\textquotedbl{} means X was NOT done \textemdash{} disqualify it. \textquotedbl{}Joined X\textquotedbl{} or \textquotedbl{}has done X\textquotedbl{} means it WAS done \textemdash{} prefer it.\par
\noindent - When multiple memories repeat the same generic fact, that repetition does NOT make it more correct than a single memory with a more specific answer.\par
\noindent - Photos depict what was IN the photo, not facts about someone\textquotesingle{}s daily life. Prefer direct statements over photo descriptions for inferences.\par
\noindent - Re-read the question carefully before answering. If it asks \textquotedbl{}what aspect/type/kind\textquotedbl{}, answer with the specific aspect. If it asks \textquotedbl{}what did they discover they both enjoy\textquotedbl{}, answer with the specific thing, not the setting.\par
\vspace{0.25\baselineskip}
\noindent \#\# Step 5: TEMPORAL GROUNDING\par
\noindent These conversations took place around \{reference\_date\}. All events occurred in 2022-2024.\par
\noindent - Calculate time relative to this date, NOT today. Never output 2025 or 2026.\par
\noindent - Use dates explicitly stated in memory text. Do not invent or estimate dates.\par
\noindent - When a question asks what someone \textquotedbl{}shared\textquotedbl{} or \textquotedbl{}mentioned\textquotedbl{} on a date, that date is when they TALKED about it \textemdash{} look for events shortly BEFORE that date.\par
\noindent - For \textquotedbl{}how long\textquotedbl{} questions, find the start and end dates explicitly, then compute the duration. Do not guess.\par
\noindent - TEMPORAL DISAMBIGUATION: When you find MULTIPLE instances of similar events at different dates, enumerate them all with their dates before picking. If the question uses past tense + \textquotedbl{}the\textquotedbl{} \ensuremath{\rightarrow} select the instance closest to (and before) the reference date. If future tense (\textquotedbl{}plans to\textquotedbl{}, \textquotedbl{}going to\textquotedbl{}) \ensuremath{\rightarrow} select the earliest planned date. NEVER default to the first-mentioned or highest-scored instance \textemdash{} the DATE determines the answer.\par
\vspace{0.25\baselineskip}
\noindent \#\# Step 6: INCLUSION CHECK (for lists and counts)\par
\noindent If you found items during reasoning that you\textquotesingle{}re tempted to exclude from your answer \textemdash{} STOP. Include them unless you have STRONG evidence they are wrong. The most common mistake is finding relevant items but then dropping them due to overly strict filtering. More items is better than fewer when there is supporting evidence.\par
\noindent - For counting: after enumerating, re-verify each item. Check for duplicates (same event described differently) and ensure you haven\textquotesingle{}t missed items from memories late in the list.\par
\noindent - The question assumes something happened. Find WHAT happened, don\textquotesingle{}t say nothing happened.\par
\vspace{0.25\baselineskip}
\noindent \#\# Step 7: COMMIT AND ANSWER\par
\noindent Give a direct, specific answer. NEVER say \textquotedbl{}not specified\textquotedbl{}, \textquotedbl{}not mentioned\textquotedbl{}, \textquotedbl{}no record\textquotedbl{}, or \textquotedbl{}the memories don\textquotesingle{}t say\textquotedbl{} \textemdash{} if ANY memory contains relevant information, give the best answer from available evidence. No hedging, no caveats. If the question asks for a list, include ALL items found. NEVER return an empty answer when relevant memories exist.\par
\noindent - NEVER generate specific names, titles, places, or dates that do not appear in any memory above. If no memory contains the specific detail the question asks for, answer with what the memories DO contain rather than guessing.\par
\noindent - For open-domain/opinion questions (\textquotedbl{}Would X do Y?\textquotedbl{}, \textquotedbl{}Is X considered Z?\textquotedbl{}):\par
\noindent \hspace*{2\fontdimen2\font}* Follow the DIRECT causal reasoning in the memories. Do NOT construct elaborate counter-arguments.\par
\noindent \hspace*{2\fontdimen2\font}* \textquotedbl{}Would X still do Y without Z?\textquotedbl{} \textemdash{} If memories show X does Y BECAUSE of Z, then without Z, answer \textquotedbl{}likely no.\textquotedbl{}\par
\noindent \hspace*{2\fontdimen2\font}* \textquotedbl{}Would X do Y again soon?\textquotedbl{} \textemdash{} If the most recent attempt involved a bad experience (accident, scare, trauma), answer \textquotedbl{}likely no.\textquotedbl{} A recent negative experience outweighs historical positive patterns.\par
\noindent \hspace*{2\fontdimen2\font}* For trait questions (\textquotedbl{}Is X considered Z?\textquotedbl{}): weigh ALL evidence including symbolic/indirect references. If there is SOME but not strong evidence, answer with a qualified degree (\textquotedbl{}somewhat\textquotedbl{}) rather than flat \textquotedbl{}no.\textquotedbl{}\par
\vspace{0.25\baselineskip}
\noindent \# Instructions\par
\vspace{0.25\baselineskip}
\noindent \#\# Misc\par
\vspace{0.25\baselineskip}
\noindent 1. Make reasonable deductions based on your memories. Memory shows store with a lot of working people -\textgreater{} store employs a lot of people\par
\noindent 2. If a memory describes something recognizable (e.g., \textquotedbl{}romantic drama about memory and relationships\textquotedbl{}), you may name it (e.g., \textquotedbl{}Eternal Sunshine of the Spotless Mind\textquotedbl{}).\par
\noindent 3. Use domain knowledge to connect facts: a game exclusive to one platform implies ownership of that platform. An unnamed company deal can be linked to a previously expressed brand preference.\par
\vspace{0.25\baselineskip}
\noindent \{memories\}\par
\vspace{0.25\baselineskip}
\noindent Question: \{question\}\par
\vspace{0.25\baselineskip}
\noindent Work through Steps 1-7, then give your final answer after \textquotedbl{}ANSWER:\textquotedbl{}.\par
\end{benchmarkprompt}

The same correctness judge is used for categories 1--4; the category argument does not change the template. The system prompt is empty. The judge receives the question, gold answer, and generated answer through the placeholders below. For open-domain questions, the preprocessing function retains the portion of the gold answer before the first semicolon.

\begin{benchmarkprompt}{judgeprompt}{LoCoMo --- Answer correctness judge}
\noindent Your task is to label an answer to a question as \textquoteright{}CORRECT\textquoteright{} or \textquoteright{}WRONG\textquoteright{}. You will be given the following data:\par
\noindent \hspace*{4\fontdimen2\font}(1) a question (posed by one user to another user),\par
\noindent \hspace*{4\fontdimen2\font}(2) a \textquoteright{}gold\textquoteright{} (ground truth) answer,\par
\noindent \hspace*{4\fontdimen2\font}(3) a generated answer\par
\noindent which you will score as CORRECT/WRONG.\par
\vspace{0.25\baselineskip}
\noindent The point of the question is to ask about something one user should know about the other user based on their prior conversations.\par
\noindent The gold answer will usually be a concise and short answer that includes the referenced topic, for example:\par
\noindent Question: Do you remember what I got the last time I went to Hawaii?\par
\noindent Gold answer: A shell necklace\par
\noindent The generated answer might be much longer, but you should be generous with your grading - as long as it touches on the same topic as the gold answer, it should be counted as CORRECT.\par
\vspace{0.25\baselineskip}
\noindent For time related questions, the gold answer will be a specific date, month, year, etc. The generated answer might be much longer or use relative time references (like \textquotedbl{}last Tuesday\textquotedbl{} or \textquotedbl{}next month\textquotedbl{}), but you should be generous with your grading - as long as it refers to the same date or time period as the gold answer, it should be counted as CORRECT. Even if the format differs (e.g., \textquotedbl{}May 7th\textquotedbl{} vs \textquotedbl{}7 May\textquotedbl{}), consider it CORRECT if it\textquotesingle{}s the same date.\par
\vspace{0.25\baselineskip}
\noindent Now it\textquotesingle{}s time for the real question:\par
\noindent Question: \{question\}\par
\noindent Gold answer: \{gold\_answer\}\par
\noindent Generated answer: \{generated\_answer\}\par
\vspace{0.25\baselineskip}
\noindent First, provide a short (one sentence) explanation of your reasoning, then finish with CORRECT or WRONG.\par
\noindent Do NOT include both CORRECT and WRONG in your response, or it will break the evaluation script.\par
\vspace{0.25\baselineskip}
\noindent Just return the label CORRECT or WRONG in a json format with the key as \textquotedbl{}label\textquotedbl{}.\par
\end{benchmarkprompt}

\paragraph{Optional source-evidence branch.}
The default \texttt{get\_judge\_prompt} call supplies no source evidence. The separate \texttt{get\_judge\_prompt\_with\_evidence} entry point inserts the following component immediately before ``Now it's time for the real question:'' in the judge template.

\begin{benchmarkprompt}{auxprompt}{LoCoMo --- Optional judge evidence component}
\noindent \#\# Evidence (actual conversation messages containing the answer)\par
\noindent \{evidence\_context\}\par
\end{benchmarkprompt}

\subsubsection{LongMemEval}
\label{app:prompts-longmemeval}

For LongMemEval, we use the same answer-generation and answer-correctness judging prompts as the Mem0 implementation~\citep{chhikara2025mem0}.

\subsubsection{BEAM}
\label{app:prompts-beam}

For BEAM, we use prompts consistent with the official paper and code implementation~\citep{tavakoli2026beam}.

\endgroup

\clearpage
\clearpage
\section{Theoretical Analysis}
\label{app:theory}

This appendix proves Proposition~\ref{prop:redundancy-coverage} and derives an optimality certificate for the same selection--refit procedure. Throughout, the query $q$ and candidate pool $\mathcal{C}$ are fixed, embeddings have unit norm, and each refit solves Eq.~\ref{eq:refit} exactly with a fixed $\lambda\ge0$. Selected memories remain in the returned context even when their fitted weights become zero.

\subsection{Setup and restricted optimality}
\label{app:theory-setup}

Index coefficient vectors by the $P$ candidate memories and define the full simplex and its restriction to a nonempty set $S$:
\begin{equation}
\Delta_{\mathcal{C}}=\left\{\boldsymbol{\alpha}\in\mathbb{R}^{P}:\boldsymbol{\alpha}\succeq0,\ \sum_{m\in\mathcal{C}}\alpha_m=1\right\},
\qquad
\Delta_{\mathcal{C}}(S)=\left\{\boldsymbol{\alpha}\in\Delta_{\mathcal{C}}:\alpha_m=0\text{ for }m\notin S\right\}.
\label{eq:theory-simplex}
\end{equation}
Write the reconstruction objective as
\begin{equation}
F_\lambda(\boldsymbol{\alpha};q)
=\frac12\left\lVert f(q)-\sum_{m\in\mathcal{C}}\alpha_m f(m)\right\rVert_2^2
+\frac\lambda2\lVert\boldsymbol{\alpha}\rVert_2^2,
\qquad
J_\lambda(S;q)=\min_{\boldsymbol{\alpha}\in\Delta_{\mathcal{C}}(S)}F_\lambda(\boldsymbol{\alpha};q).
\label{eq:theory-objective}
\end{equation}
Let $\boldsymbol{\alpha}^{S}$ be a minimizing coefficient vector, with entries $\alpha_m(q,S)$ on $S$ and zeros outside $S$. For $S=S_t$, it is the zero-padded version of the refit vector $\boldsymbol{\alpha}_t$ in Eq.~\ref{eq:refit}. A minimizer exists by compactness and is unique for $\lambda>0$; for $\lambda=0$, any minimizer may be used. To shorten the proof, write $\hat{\mathbf{q}}_S=\hat{\mathbf{q}}(q,S)$, $\mathbf{r}_S=\mathbf{r}(q,S)$, and
\begin{equation}
U_{\mathrm{MP}}(m\mid q,S)=f(m)^{\top}\mathbf{r}_S,
\qquad
\nu_S=\hat{\mathbf{q}}_S^{\top}\mathbf{r}_S-\lambda\lVert\boldsymbol{\alpha}^{S}\rVert_2^2.
\label{eq:theory-scores}
\end{equation}
Since $\hat{\mathbf{q}}_S$ is a convex combination of unit vectors, $\lVert\hat{\mathbf{q}}_S\rVert_2\le1$ and $\lVert\mathbf{r}_S\rVert_2\le2$.

\begin{lemma}[Restricted optimality]
\label{lem:restricted-kkt}
For every nonempty $S$, there are multipliers $\mu_m\ge0$ such that
\begin{equation}
U_{\mathrm{MP}}(m\mid q,S)=\nu_S+\lambda\alpha_m^{S}-\mu_m,
\qquad \mu_m\alpha_m^{S}=0\quad(m\in S).
\label{eq:restricted-kkt}
\end{equation}
In particular,
\begin{equation}
\max_{m\in S}U_{\mathrm{MP}}(m\mid q,S)=\nu_S+\lambda\lVert\boldsymbol{\alpha}^{S}\rVert_\infty.
\label{eq:selected-score-max}
\end{equation}
\end{lemma}

\begin{proof}
On the coordinates in $S$, the Lagrangian is
\[
\mathcal{L}(\boldsymbol{\alpha},\nu,\boldsymbol{\mu})
=F_\lambda(\boldsymbol{\alpha};q)+\nu\left(\sum_{m\in S}\alpha_m-1\right)-\sum_{m\in S}\mu_m\alpha_m.
\]
The restricted simplex has a strictly positive feasible vector. Convexity and the KKT conditions give
$-U_{\mathrm{MP}}(m\mid q,S)+\lambda\alpha_m^{S}+\nu-\mu_m=0$, with $\mu_m\ge0$ and $\mu_m\alpha_m^{S}=0$. Multiplying by $\alpha_m^{S}$ and summing yields $\nu=\nu_S$, proving Eq.~\ref{eq:restricted-kkt}. Every selected score is at most $\nu_S+\lambda\lVert\boldsymbol{\alpha}^{S}\rVert_\infty$. A coordinate attaining the largest weight is positive, so its multiplier vanishes and its score attains this bound.
\end{proof}

\subsection{Fact-cluster model}
\label{app:fact-model}

The following model links memories supplying the same required fact to their embedding geometry. Its labels and parameters are used for analysis; selection uses only the embeddings and fitted residual.

\begin{assumption}[Fact-cluster separation and coverage margin]
\label{ass:fact-clusters}
The query requires $h\ge2$ atomic facts, and the candidate memories admit a disjoint partition
\[
\mathcal{C}=\mathcal{C}_1\mathbin{\dot\cup}\cdots\mathbin{\dot\cup}\mathcal{C}_h\mathbin{\dot\cup}\mathcal{D},
\qquad \mathcal{R}=\bigcup_{\ell=1}^{h}\mathcal{C}_\ell.
\]
Each $\mathcal{C}_\ell$ is nonempty and contains alternative memories supplying the same required fact; one representative suffices to make that fact available in the context. The remaining candidates form the distractor set $\mathcal{D}$. For $0\le\rho<1$, $\varepsilon,\eta\ge0$, $a_{\min}\le a_{\max}$, and $\gamma>0$, assume:
\begin{enumerate}
\item \textbf{Within-fact concentration:} $\lVert f(m)-f(s)\rVert_2\le\varepsilon$ for $m,s\in\mathcal{C}_\ell$.
\item \textbf{Between-fact similarity:} $|f(m)^{\top}f(s)-\rho|\le\eta$ for $m\in\mathcal{C}_\ell$, $s\in\mathcal{C}_k$, $\ell\ne k$.
\item \textbf{Balanced query relevance:} $a_{\min}\le f(q)^{\top}f(m)\le a_{\max}$ for $m\in\mathcal{R}$.
\item \textbf{Distractor separation:} $f(q)^{\top}f(d)\le a_{\min}-\gamma$ and $f(d)^{\top}f(m)\ge\rho-\eta$ for $d\in\mathcal{D}$, $m\in\mathcal{R}$.
\end{enumerate}
Write $\Delta_a=a_{\max}-a_{\min}$. The coverage margin satisfies
\begin{equation}
M:=\frac{1-\rho}{h-1}-\Delta_a-2\eta-2\varepsilon>0,
\qquad \gamma>2\eta\ \text{when distractors are present}.
\label{eq:coverage-margin}
\end{equation}
When $\mathcal{D}=\varnothing$, conditions involving $\gamma$ are omitted.
\end{assumption}

\subsection{Proof of Proposition~\ref{prop:redundancy-coverage}}
\label{app:coverage-proof}

\begin{proof}
\textbf{Geometric redundancy.}
The finite convex hull $H_S$ is compact. If $\operatorname{dist}(f(m),H_S)\le\delta$, choose $\mathbf{v}=\sum_{s\in S}\beta_s f(s)\in H_S$ with $\lVert f(m)-\mathbf{v}\rVert_2\le\delta$, where $\beta_s\ge0$ and $\sum_s\beta_s=1$. Then
\begin{align}
U_{\mathrm{MP}}(m\mid q,S)
&=\sum_{s\in S}\beta_s U_{\mathrm{MP}}(s\mid q,S)+(f(m)-\mathbf{v})^{\top}\mathbf{r}_S\notag\\
&\le\max_{s\in S}U_{\mathrm{MP}}(s\mid q,S)+\delta\lVert\mathbf{r}_S\rVert_2\notag\\
&=\nu_S+\lambda\lVert\boldsymbol{\alpha}^{S}\rVert_\infty+\delta\lVert\mathbf{r}_S\rVert_2.
\label{eq:theory-redundancy}
\end{align}
The inequality uses Cauchy--Schwarz and the equality uses Lemma~\ref{lem:restricted-kkt}. This proves Eq.~\ref{eq:redundancy-bound}. If an unselected candidate exceeds this bound, no memory within distance $\delta$ of $H_S$ can maximize the selection score.

\textbf{Coverage induction.}
At the first step, each required-fact memory has query similarity at least $a_{\min}$, whereas every distractor has similarity at most $a_{\min}-\gamma<a_{\min}$. Thus the first selection belongs to a required-fact cluster. If there are no distractors, this holds directly.

Suppose $S=S_t$ contains one representative from each of $t<h$ distinct fact clusters and no distractor. Let $G_{SS}$ be the selected memory Gram matrix. Assumption~\ref{ass:fact-clusters} and unit norms give
\begin{equation}
G_{SS}=(1-\rho)I_t+\rho\mathbf{1}\mathbf{1}^{\top}+E,
\qquad E_{ss}=0,\quad |E_{ms}|\le\eta.
\label{eq:fact-gram}
\end{equation}
Write $\boldsymbol{\alpha}=\boldsymbol{\alpha}^{S}|_S$. Simplex feasibility gives $|(E\boldsymbol{\alpha})_s|\le\eta$ and $\alpha_{\max}\ge1/t$. Choose a selected memory $j$ attaining $\alpha_{\max}$. By Lemma~\ref{lem:restricted-kkt}, it also attains the largest selected residual score, so
\begin{align}
\max_{s\in S}U_{\mathrm{MP}}(s\mid q,S)
&=f(q)^{\top}f(j)-\rho-(1-\rho)\alpha_{\max}-(E\boldsymbol{\alpha})_j\notag\\
&\le a_{\max}-\rho+\eta-\frac{1-\rho}{t}.
\label{eq:covered-score-max}
\end{align}

For an unselected memory $m$ from an already represented cluster, its selected representative lies within distance $\varepsilon$. Equation~\ref{eq:theory-redundancy} and $\lVert\mathbf{r}_S\rVert_2\le2$ therefore imply
\begin{equation}
U_{\mathrm{MP}}(m\mid q,S)\le a_{\max}-\rho+\eta-\frac{1-\rho}{t}+2\varepsilon.
\label{eq:covered-candidate-bound}
\end{equation}
For a memory $u$ from an uncovered cluster, every selected memory belongs to a different cluster. Hence $f(u)^{\top}\hat{\mathbf{q}}_S\le\rho+\eta$ and
\begin{equation}
\begin{aligned}
U_{\mathrm{MP}}(u\mid q,S)&\ge a_{\min}-\rho-\eta,\\
U_{\mathrm{MP}}(u\mid q,S)-U_{\mathrm{MP}}(m\mid q,S)
&\ge\frac{1-\rho}{t}-\Delta_a-2\eta-2\varepsilon\ge M>0.
\end{aligned}
\label{eq:uncovered-score-margin}
\end{equation}
For a distractor $d$, the separation assumption gives $f(d)^{\top}\hat{\mathbf{q}}_S\ge\rho-\eta$, yielding
\begin{equation}
\begin{aligned}
U_{\mathrm{MP}}(d\mid q,S)&\le a_{\min}-\gamma-\rho+\eta,\\
U_{\mathrm{MP}}(u\mid q,S)-U_{\mathrm{MP}}(d\mid q,S)&\ge\gamma-2\eta>0.
\end{aligned}
\label{eq:distractor-score-margin}
\end{equation}
An uncovered cluster exists because $t<h$. Every memory from such a cluster strictly outranks both redundant candidates and distractors, so the next selection represents a new required fact. Induction gives one representative from each of the $h$ clusters in the first $h$ selections. Since selected memories are retained, coverage persists at every larger returned prefix.
\end{proof}

The coverage conclusion concerns the availability of required facts in the selected memory texts. It holds for every fixed $\lambda\ge0$ under Eq.~\ref{eq:coverage-margin}, while allowing the fitted weights and selection order to depend on $\lambda$.

\subsection{A computable optimality certificate}
\label{app:optimality-certificate}

The following certificate specializes the Frank--Wolfe gap~\citep{jaggi2013revisiting} to the ridge-regularized simplex refit. It requires only the setup in Appendix~\ref{app:theory-setup}. Define the optimum over the full candidate simplex by
\begin{equation}
F_\lambda^{\star}(q)=\min_{\boldsymbol{\alpha}\in\Delta_{\mathcal{C}}}F_\lambda(\boldsymbol{\alpha};q).
\label{eq:full-simplex-optimum}
\end{equation}

\begin{theorem}[Computable optimality certificate]
\label{thm:optimality-certificate}
For a nonempty selected set $S$ with exact refitting, define
\begin{equation}
\begin{aligned}
g_S&=\left[\max_{m\in\mathcal{C}\setminus S}U_{\mathrm{MP}}(m\mid q,S)-\nu_S\right]_+\\
&=\left[\max_{m\in\mathcal{C}\setminus S}f(m)^{\top}\mathbf{r}_S
-\hat{\mathbf{q}}_S^{\top}\mathbf{r}_S+\lambda\lVert\boldsymbol{\alpha}^{S}\rVert_2^2\right]_+,
\end{aligned}
\label{eq:optimality-gap}
\end{equation}
where $[a]_+=\max\{a,0\}$ and $\max\varnothing=-\infty$. Then $g_S$ is the Frank--Wolfe gap of $F_\lambda(\cdot;q)$ on $\Delta_{\mathcal{C}}$ at $\boldsymbol{\alpha}^{S}$, and
\begin{equation}
0\le J_\lambda(S;q)-F_\lambda^{\star}(q)\le g_S.
\label{eq:optimality-certificate}
\end{equation}
Moreover, $g_S=0$ if and only if $\boldsymbol{\alpha}^{S}$ minimizes the objective over the full simplex.
\end{theorem}

\begin{proof}
Let $\mathbf{p}=\nabla_{\boldsymbol{\alpha}}F_\lambda(\boldsymbol{\alpha}^{S};q)$. Its coordinates satisfy $p_m=-U_{\mathrm{MP}}(m\mid q,S)+\lambda\alpha_m^{S}$, and $(\boldsymbol{\alpha}^{S})^{\top}\mathbf{p}=-\nu_S$. By Lemma~\ref{lem:restricted-kkt}, $p_m\ge-\nu_S$ on $S$, with equality at every positive-weight coordinate. At least one such coordinate exists, so $\min_{m\in S}p_m=-\nu_S$. Outside $S$, the weights vanish and $p_m=-U_{\mathrm{MP}}(m\mid q,S)$. Consequently,
\begin{align}
g_{\mathrm{FW}}(\boldsymbol{\alpha}^{S})
&:=\max_{\boldsymbol{\beta}\in\Delta_{\mathcal{C}}}\langle\mathbf{p},\boldsymbol{\alpha}^{S}-\boldsymbol{\beta}\rangle\notag\\
&=(\boldsymbol{\alpha}^{S})^{\top}\mathbf{p}-\min_{m\in\mathcal{C}}p_m\notag\\
&=-\nu_S-\min\left\{-\nu_S,-\max_{m\in\mathcal{C}\setminus S}U_{\mathrm{MP}}(m\mid q,S)\right\}
=g_S.
\label{eq:gap-identity}
\end{align}
Let $\boldsymbol{\alpha}^{\star}$ be a full-simplex minimizer. Convexity gives
\[
F_\lambda(\boldsymbol{\alpha}^{\star};q)
\ge F_\lambda(\boldsymbol{\alpha}^{S};q)
+\langle\mathbf{p},\boldsymbol{\alpha}^{\star}-\boldsymbol{\alpha}^{S}\rangle.
\]
Rearranging and maximizing the resulting inner product over the simplex proves Eq.~\ref{eq:optimality-certificate}. Thus $g_S=0$ implies optimality. Conversely, first-order optimality at a full-simplex minimizer gives $\langle\mathbf{p},\boldsymbol{\beta}-\boldsymbol{\alpha}^{S}\rangle\ge0$ for every feasible $\boldsymbol{\beta}$; choosing $\boldsymbol{\beta}=\boldsymbol{\alpha}^{S}$ shows that the maximum defining the gap is zero.
\end{proof}

The certificate bounds the remaining error in the regularized reconstruction objective over the current candidate pool. It is computed from the residual scores and fitted weights already used by \method.

\clearpage
\section{Additional Experimental Results}
\label{app:additional-results}

\FloatBarrier
\subsection{Full cross-framework results}
\label{app:cross-framework-results}

Table~\ref{tab:cross-system-read-full} reports the complete results at $K=10$ and $K=20$ for the cross-framework comparisons in Table~\ref{tab:cross-system-read}.

\begin{table}[htbp]
\caption{Full cross-framework comparisons on LoCoMo at $K\in\{10,20\}$. Each system is evaluated without (w/o) and with (w/) \method at the same selection budget $K$ (seed count for MemoryBank). Entries report accuracy (\%); parentheses in the w/ rows give the change in percentage points from the paired w/o result. For each system and budget, the higher accuracy in each metric is bolded. Yellow highlights the w/ results. All configurations evaluate 1,540 questions using GPT-4o-mini as answerer and judge.}
\label{tab:cross-system-read-full}
\centering
\small
\arrayrulecolor{black}
\setlength{\tabcolsep}{2.5pt}
\renewcommand{\arraystretch}{1.25}
\newcommand{\memgain}[2]{#1\,{\scriptsize($#2$)}}
\newcommand{\membase}[1]{#1\phantom{\,{\scriptsize($+0.00$)}}}
\begin{tabular}{@{}c|l|clllll}
\hline
$K$ & System & {\scriptsize\method} & Multi-hop & Temporal & Open-domain & Single-hop & Overall \\
\hline
\multirow{8}{*}{$10$} & \multirow{2}{*}{LightMem} & w/o & \membase{62.77} & \membase{66.98} & \membase{56.25} & \membase{77.17} & \membase{71.10} \\
 &  & w/ & \cellcolor{memoryyellow}\memgain{\textbf{69.15}}{+6.38} & \cellcolor{memoryyellow}\memgain{\textbf{68.22}}{+1.25} & \cellcolor{memoryyellow}\memgain{\textbf{57.29}}{+1.04} & \cellcolor{memoryyellow}\memgain{\textbf{79.19}}{+2.02} & \cellcolor{memoryyellow}\memgain{\textbf{73.70}}{+2.60} \\
\hhline{~|~|------}
 & \multirow{2}{*}{MemoryBank} & w/o & \membase{53.55} & \membase{\textbf{43.30}} & \membase{51.04} & \membase{65.87} & \membase{57.99} \\
 &  & w/ & \cellcolor{memoryyellow}\memgain{\textbf{56.74}}{+3.19} & \cellcolor{memoryyellow}\memgain{39.25}{-4.05} & \cellcolor{memoryyellow}\memgain{\textbf{56.25}}{+5.21} & \cellcolor{memoryyellow}\memgain{\textbf{70.27}}{+4.40} & \cellcolor{memoryyellow}\memgain{\textbf{60.45}}{+2.47} \\
\hhline{~|~|------}
 & \multirow{2}{*}{SimpleMem} & w/o & \membase{62.41} & \membase{48.60} & \membase{56.25} & \membase{74.91} & \membase{65.97} \\
 &  & w/ & \cellcolor{memoryyellow}\memgain{\textbf{70.21}}{+7.80} & \cellcolor{memoryyellow}\memgain{\textbf{49.22}}{+0.62} & \cellcolor{memoryyellow}\memgain{\textbf{60.42}}{+4.17} & \cellcolor{memoryyellow}\memgain{\textbf{82.28}}{+7.37} & \cellcolor{memoryyellow}\memgain{\textbf{71.82}}{+5.84} \\
\hhline{~|~|------}
 & \multirow{2}{*}{Mem0} & w/o & \membase{45.04} & \membase{42.37} & \membase{43.75} & \membase{57.43} & \membase{51.17} \\
 &  & w/ & \cellcolor{memoryyellow}\memgain{\textbf{55.32}}{+10.28} & \cellcolor{memoryyellow}\memgain{\textbf{49.84}}{+7.48} & \cellcolor{memoryyellow}\memgain{\textbf{48.96}}{+5.21} & \cellcolor{memoryyellow}\memgain{\textbf{67.54}}{+10.11} & \cellcolor{memoryyellow}\memgain{\textbf{60.45}}{+9.29} \\
\hline
\multirow{8}{*}{$20$} & \multirow{2}{*}{LightMem} & w/o & \membase{73.40} & \membase{\textbf{70.09}} & \membase{\textbf{63.54}} & \membase{83.12} & \membase{77.40} \\
 &  & w/ & \cellcolor{memoryyellow}\memgain{\textbf{74.11}}{+0.71} & \cellcolor{memoryyellow}\memgain{\textbf{70.09}}{+0.00} & \cellcolor{memoryyellow}\memgain{61.46}{-2.08} & \cellcolor{memoryyellow}\memgain{\textbf{83.71}}{+0.59} & \cellcolor{memoryyellow}\memgain{\textbf{77.73}}{+0.32} \\
\hhline{~|~|------}
 & \multirow{2}{*}{MemoryBank} & w/o & \membase{59.22} & \membase{\textbf{48.91}} & \membase{\textbf{54.17}} & \membase{72.89} & \membase{64.22} \\
 &  & w/ & \cellcolor{memoryyellow}\memgain{\textbf{64.18}}{+4.96} & \cellcolor{memoryyellow}\memgain{47.66}{-1.25} & \cellcolor{memoryyellow}\memgain{\textbf{54.17}}{+0.00} & \cellcolor{memoryyellow}\memgain{\textbf{78.72}}{+5.83} & \cellcolor{memoryyellow}\memgain{\textbf{68.05}}{+3.83} \\
\hhline{~|~|------}
 & \multirow{2}{*}{SimpleMem} & w/o & \membase{70.21} & \membase{52.96} & \membase{\textbf{67.71}} & \membase{79.90} & \membase{71.75} \\
 &  & w/ & \cellcolor{memoryyellow}\memgain{\textbf{77.66}}{+7.45} & \cellcolor{memoryyellow}\memgain{\textbf{54.83}}{+1.87} & \cellcolor{memoryyellow}\memgain{64.58}{-3.12} & \cellcolor{memoryyellow}\memgain{\textbf{85.97}}{+6.06} & \cellcolor{memoryyellow}\memgain{\textbf{76.62}}{+4.87} \\
\hhline{~|~|------}
 & \multirow{2}{*}{Mem0} & w/o & \membase{56.38} & \membase{47.04} & \membase{\textbf{51.04}} & \membase{63.85} & \membase{58.18} \\
 &  & w/ & \cellcolor{memoryyellow}\memgain{\textbf{58.51}}{+2.13} & \cellcolor{memoryyellow}\memgain{\textbf{51.71}}{+4.67} & \cellcolor{memoryyellow}\memgain{50.00}{-1.04} & \cellcolor{memoryyellow}\memgain{\textbf{69.32}}{+5.47} & \cellcolor{memoryyellow}\memgain{\textbf{62.47}}{+4.29} \\
\hline
\end{tabular}
\end{table}

\FloatBarrier
\subsection{Reconstruction Residuals and Gold Evidence}
\label{app:residual-analysis}

\begin{wrapfigure}{l}{0.42\textwidth}
\vspace{-\intextsep}
\captionsetup{font=footnotesize}
\centering
\includegraphics[width=\linewidth,trim=0 230bp 0 0,clip]{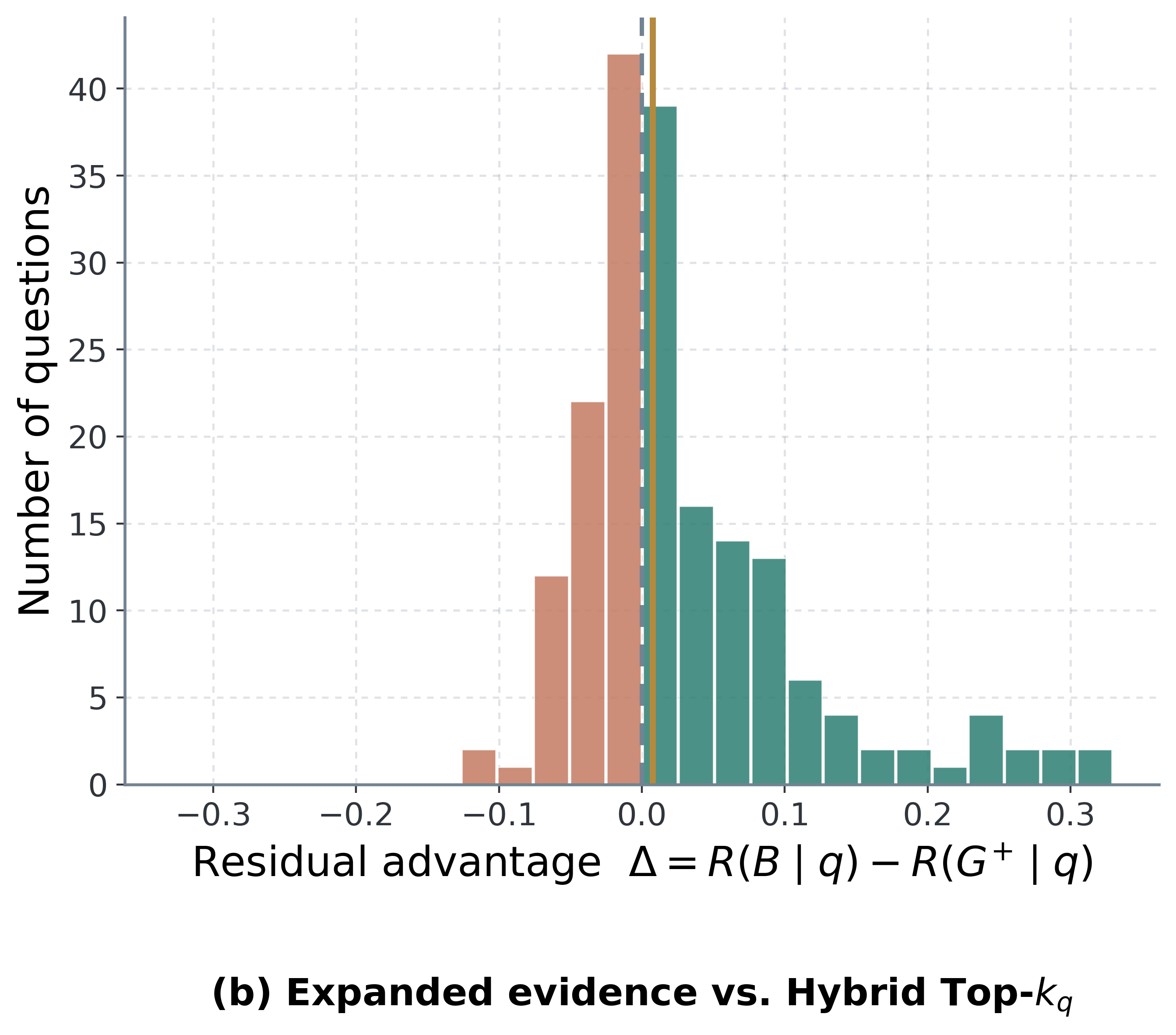}
\caption{Reconstruction residual advantage of gold evidence over hybrid retrieval at matched set size.  Positive $\Delta_q$ (green) indicates a lower residual norm for the gold evidence set; negative $\Delta_q$ (orange) indicates a lower residual for the baseline. The dashed line marks equal residuals. Both sets are evaluated with the same simplex refit.}
\label{fig:gold-residual-advantage}
\end{wrapfigure}

\textbf{Reconstruction provides a useful target for evidence composition.} To connect the evidence-coverage gains in Section~\ref{sec:mechanism-results} to the selection objective, we compare the reconstruction residuals of gold evidence and hybrid retrieval at equal set size. For each query $q$, let $G_q^{+}$ be its gold evidence memory set and $B_q$ the hybrid Top-$K_q^{\mathrm{gold}}$ set, where $K_q^{\mathrm{gold}}=|G_q^{+}|$. We keep the gold set fixed and match the baseline budget to it. We evaluate both sets using the refitted residual norm $R(S\mid q)=\lVert\mathbf{r}(q,S)\rVert_2$, testing whether gold evidence yields a closer fit in the same representation space that supplies the residual-guided utility in Eq.~\ref{eq:memory-utility}. We report the paired residual advantage
\begin{equation}
\Delta_q=R(B_q\mid q)-R(G_q^{+}\mid q),
\label{eq:residual-advantage}
\end{equation}
so positive values indicate that gold evidence reconstructs the query more closely than the equally sized hybrid selection.

Figure~\ref{fig:gold-residual-advantage} shows that gold evidence achieves lower residuals than hybrid Top-$K_q^{\mathrm{gold}}$ on many queries, with pronounced advantages on the positive side of the distribution. At equal set size, this advantage links the content of the evidence to how closely its embeddings reconstruct the query. It provides an empirical basis for \method's central mechanism: refitting accounts for what the selected memories already represent, and the residual directs subsequent recall toward query content that remains uncovered. Together with the coverage gains, this analysis connects the geometry used for selection to the complementary evidence recovered in the answer context.

\FloatBarrier
\subsection{Residual preference for missing evidence}
\label{app:missing-evidence-utility}

\begin{wrapfigure}{l}{0.45\textwidth}
\vspace{-\intextsep}
\captionsetup{font=footnotesize}
\centering
\includegraphics[width=\linewidth]{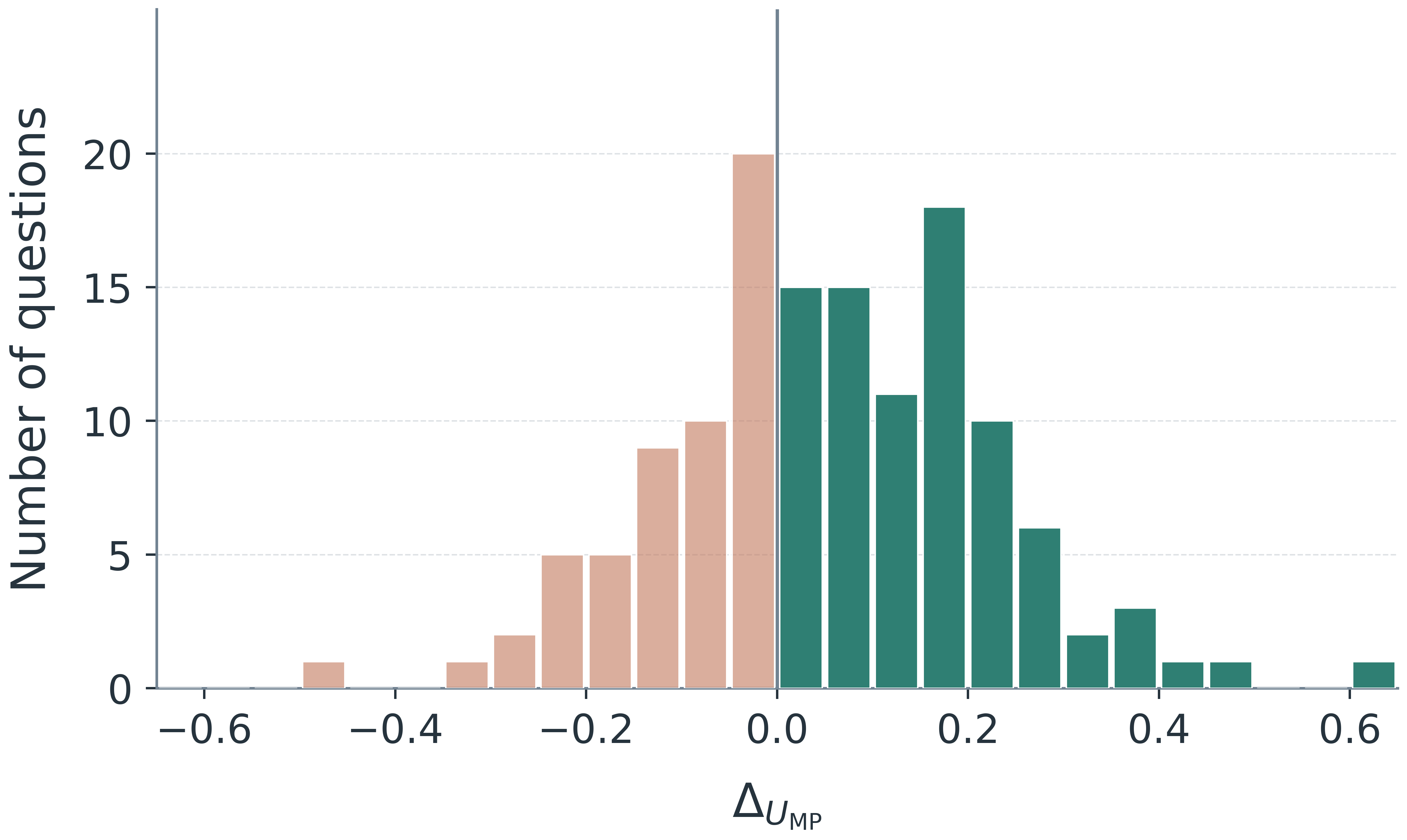}
\caption{Residual preference for missing evidence on LongMemEval with $|S|=1$. The histogram contains one mean utility difference per question (136 questions, 381 singleton comparisons). Positive values (green) favor remaining evidence over the semantic comparator, while negative values (orange) favor the comparator. The gray line marks zero.}
\label{fig:missing-evidence-utility}
\end{wrapfigure}

We test whether the residual score favors missing evidence once an answer-bearing memory is already selected. We reuse saved $P=200$ candidate pools from LongMemEval for multi-session, temporal-reasoning, and knowledge-update questions with at least two annotated answer sessions, all present in the pool. Because session annotations do not identify which individual memories contain answer facts, a local Qwen3.5-4B judge receives the question, reference answer, and candidate memories from the annotated sessions. A memory is accepted as evidence only if the judge identifies a required answer fact and supplies a supporting quotation that occurs verbatim in the memory. We retain 136 questions with evidence in at least two answer sessions and missing evidence remaining after duplicate-fact filtering.

For each question, we rotate the accepted evidence memories through the singleton selected set $S^{(j)}=\{s^{(j)}\}$, yielding 381 comparisons in total. We remove evidence repeating the selected fact and keep one randomly sampled representative per remaining duplicate-fact group, using supporting-quotation overlap and embedding similarity to identify duplicates. The semantic comparator $m_{\mathrm{sem}}^{(j)}$ is the highest-query-cosine candidate after excluding $s^{(j)}$ and memories supplying an unselected answer fact. It may therefore repeat the selected fact or be a candidate not accepted as evidence. We compute $U_{\mathrm{MP}}$ using Eq.~\ref{eq:memory-utility}. Remaining-evidence scores are averaged within each answer session and then equally across sessions to obtain $U_E^{(j)}$. Each question contributes one mean difference:
\begin{equation}
\overline{\Delta}_{U_{\mathrm{MP}}}(q)
=\frac{1}{D_q}\sum_{j=1}^{D_q}\left[
U_E^{(j)}-U_{\mathrm{MP}}(m_{\mathrm{sem}}^{(j)}\mid q,S^{(j)})
\right],
\label{eq:missing-evidence-utility}
\end{equation}
where $D_q$ is the number of singleton selections for question $q$.

Figure~\ref{fig:missing-evidence-utility} shows more questions with positive than negative utility differences. Once one evidence memory is selected, residual scoring tends to favor complementary answer facts over highly query-similar alternatives. This controlled comparison supports the role of the residual in directing selection toward evidence still missing from the context.

\par\medskip\noindent\begin{minipage}{\linewidth}
\paragraph{Case study: Recovering missing evidence.}
Table~\ref{tab:missing-evidence-case} traces this behavior in a LongMemEval multi-session question. At the same budget of ten memories, hybrid retrieval includes repeated descriptions of the grandmother's 75th birthday but leaves the user's age just outside the cutoff, at rank 11. \method retains the grandmother's age and selects the user's age at step 6, completing the evidence needed to compute $75-32=43$. Both answerers then return the correct age difference. This example complements the reconstruction analysis in Appendix~\ref{app:residual-analysis} and the residual-utility comparison above by tracing missing evidence into the selected context and the resulting answer.

\captionof{table}{Recovery of a missing age fact on LongMemEval ($K=10$). Memory excerpts show the two required facts and a repeated birthday description. Hybrid positions are retrieval ranks, while \method positions are greedy selection steps. Answers are reported as age differences in years.}
\label{tab:missing-evidence-case}
\centering
\small
\setlength{\tabcolsep}{4pt}
\renewcommand{\arraystretch}{1.15}
\begin{minipage}{\linewidth}
\textbf{Query:} How many years older is my grandma than me?\par
\textbf{Gold answer:} 43.
\end{minipage}
\par\smallskip
\begin{tabular}{@{}>{\raggedright\arraybackslash}p{0.56\linewidth}>{\raggedright\arraybackslash}p{0.21\linewidth}>{\raggedright\arraybackslash}p{0.17\linewidth}@{}}
\toprule
\textbf{Memory excerpt / outcome} & \textbf{Hybrid baseline} & \textbf{\method} \\
\midrule
``User's grandma is energetic and sharp at \textbf{75 years old}, and user has been thinking about her a lot lately\ldots'' & Rank 1 (selected) & Step 1 (selected) \\
\addlinespace
``User's grandma's \textbf{75th birthday celebration} was a significant event for their family, inspiring them to plan a similar gathering.'' & Rank 6 (selected) & Outside Top-10 \\
\addlinespace
\rowcolor{green!8}
``User is \textbf{32 years old} and is reflecting on whether this age is considered young or old in the grand scheme of things.'' & \textbf{Rank 11} (outside Top-10) & \textbf{Step 6} (selected) \\
\midrule
Required age facts in the context & 1/2 & \textbf{2/2} \\
Qwen3-30B-A3B answer & 1 (incorrect) & \textbf{43} (correct) \\
GPT-4o-mini answer & 52 (incorrect) & \textbf{43} (correct) \\
\bottomrule
\end{tabular}
\end{minipage}

\par\bigskip\noindent\begin{minipage}{\linewidth}
\paragraph{Recovering the earlier event in a temporal comparison.}
In the LoCoMo example in Table~\ref{tab:missing-evidence-adoption}, hybrid retrieval includes records about Buddy and the two dogs but misses Toby's adoption in July. \method selects this earlier event at step 3, promoting it from hybrid rank 26. Together with the October record of Buddy's recent adoption, it supplies both endpoints of the approximately three-month interval, allowing Qwen3-30B-A3B to answer correctly at $K=10$.

\captionof{table}{Recovering a missing adoption event on LoCoMo (\texttt{conv5\_q34}, temporal, $K=10$). Hybrid ranks and pursuit steps refer to the same memories. The date outside quotation marks is the memory timestamp. Answers are condensed.}
\label{tab:missing-evidence-adoption}
\centering
\small
\setlength{\tabcolsep}{4pt}
\renewcommand{\arraystretch}{1.15}
\begin{minipage}{\linewidth}
\textbf{Query:} How many months passed between Andrew adopting Toby and Buddy?\par
\textbf{Gold answer:} Three months.
\end{minipage}
\par\smallskip
\begin{tabular}{@{}>{\raggedright\arraybackslash}p{0.56\linewidth}>{\raggedright\arraybackslash}p{0.21\linewidth}>{\raggedright\arraybackslash}p{0.17\linewidth}@{}}
\toprule
\textbf{Memory excerpt / outcome} & \textbf{Hybrid baseline} & \textbf{\method} \\
\midrule
``Andrew recently adopted a new dog named \textbf{Buddy} from a shelter\ldots''\par\textit{Recorded: 19 October 2023.} & Rank 7 (selected) & Step 1 (selected) \\
\addlinespace
\rowcolor{green!8}
``Andrew adopted a puppy named \textbf{Toby}, bringing him home to enjoy city living around \textbf{July 2023}.'' & \textbf{Rank 26} (outside Top-10) & \textbf{Step 3} (selected) \\
\midrule
Qwen3-30B-A3B answer & Cannot determine (incorrect) & \textbf{3 months} (correct) \\
\bottomrule
\end{tabular}
\end{minipage}

\par\bigskip\noindent\begin{minipage}{\linewidth}
\paragraph{Recovering a specific diagnosis beyond general health concerns.}
Table~\ref{tab:missing-evidence-diagnosis} illustrates a different omission. The hybrid context contains Sam's weight-related health warning but excludes the memory naming his diagnosis. \method retains the warning and promotes the gastritis record from hybrid rank 41 to step 6. The answer consequently changes from a general health risk to the specific condition named in the benchmark answer. This case concerns recovery of the annotated diagnosis, rather than establishing a medical causal relationship between weight and gastritis.

\captionof{table}{Recovering the annotated diagnosis on LoCoMo (\texttt{conv8\_q64}, multi-hop, $K=10$). Memory excerpts distinguish the general health warning from the specific diagnosis. Correctness follows the benchmark annotation.}
\label{tab:missing-evidence-diagnosis}
\centering
\small
\setlength{\tabcolsep}{4pt}
\renewcommand{\arraystretch}{1.15}
\begin{minipage}{\linewidth}
\textbf{Query:} Which ailment does Sam have to face due to his weight?\par
\textbf{Gold answer:} Gastritis.
\end{minipage}
\par\smallskip
\begin{tabular}{@{}>{\raggedright\arraybackslash}p{0.56\linewidth}>{\raggedright\arraybackslash}p{0.21\linewidth}>{\raggedright\arraybackslash}p{0.17\linewidth}@{}}
\toprule
\textbf{Memory excerpt / outcome} & \textbf{Hybrid baseline} & \textbf{\method} \\
\midrule
``Sam went for a health check-up on October 2, 2023, where his doctor informed him that his weight poses a \textbf{serious health risk}\ldots'' & Rank 4 (selected) & Step 3 (selected) \\
\addlinespace
\rowcolor{green!8}
``Sam had a health scare last weekend, ending up in the ER with a severe stomachache, which was diagnosed as \textbf{gastritis}.'' & \textbf{Rank 41} (outside Top-10) & \textbf{Step 6} (selected) \\
\midrule
Qwen3-30B-A3B answer & Serious health risk (incorrect) & \textbf{Gastritis} (correct) \\
\bottomrule
\end{tabular}
\end{minipage}

\par\bigskip\noindent\begin{minipage}{\linewidth}
\paragraph{Recovering the second date within the same source batch.}
In the BEAM example in Table~\ref{tab:missing-evidence-survey}, hybrid retrieval repeatedly selects meetup-related memories but omits the survey date at rank 12. \method selects the March 28 meetup and the March 10 survey at steps 1 and 2, respectively. Both answerers then return the correct interval of 18 days. The two memories belong to the same source batch, illustrating why source coverage alone can miss an omission of a required fact.

\captionof{table}{Recovering a missing date on BEAM-500K (conversation 32, temporal reasoning question 1, $K=10$). Both memory excerpts originate from \texttt{batch\_001}. Event dates appear in the memory text. Answers are condensed.}
\label{tab:missing-evidence-survey}
\centering
\small
\setlength{\tabcolsep}{4pt}
\renewcommand{\arraystretch}{1.15}
\begin{minipage}{\linewidth}
\textbf{Query:} How many days do I have between finalizing my survey plans at Caf\'{e} Soleil and attending the Raymondburgh Startup Meetup to prepare effectively?\par
\textbf{Gold answer:} 18 days.
\end{minipage}
\par\smallskip
\begin{tabular}{@{}>{\raggedright\arraybackslash}p{0.56\linewidth}>{\raggedright\arraybackslash}p{0.21\linewidth}>{\raggedright\arraybackslash}p{0.17\linewidth}@{}}
\toprule
\textbf{Memory excerpt / outcome} & \textbf{Hybrid baseline} & \textbf{\method} \\
\midrule
``User plans to attend the Raymondburgh Startup Meetup on \textbf{March 28, 2024}, to meet potential mentors\ldots'' & Rank 1 (selected) & Step 1 (selected) \\
\addlinespace
\rowcolor{green!8}
``User is concerned that their survey at Caf\'{e} Soleil on \textbf{March 10, 2024}, may not yield the necessary results\ldots'' & \textbf{Rank 12} (outside Top-10) & \textbf{Step 2} (selected) \\
\midrule
Qwen3-30B-A3B answer & Cannot determine (incorrect) & \textbf{18 days} (correct) \\
GPT-4o-mini answer & 30 days (incorrect) & \textbf{18 days} (correct) \\
\bottomrule
\end{tabular}
\end{minipage}

\FloatBarrier
\clearpage
\subsection{Question-Type Results at Top-20}
\label{app:question-types}

\begin{table}[!htbp]
\caption{LoCoMo question-type accuracy (\%) at Top-20 under the controlled LoCoMo evaluation setup, comparing hybrid ordering, \method, and pointwise neural rerankers.}
\label{tab:question-type-top20}
\centering
\small
\resizebox{\linewidth}{!}{%
\begin{tabular}{lccccc}
\toprule
Read policy & Multi-hop & Temporal & Open-domain & Single-hop & Overall \\
\midrule
Original hybrid order & 75.5 & 73.8 & 64.6 & 84.9 & 79.6 \\
\method & 80.9 & 77.9 & 63.5 & 90.7 & 85.1 \\
Qwen3.5-0.8B & 78.0 & 77.9 & 63.5 & 87.8 & 82.4 \\
Qwen3.5-2B & 81.6 & 79.4 & 66.7 & 91.3 & 85.5 \\
Qwen3.5-4B & 86.9 & 80.7 & 65.6 & 91.1 & 86.6 \\
BGE-base & 75.2 & 67.6 & 64.6 & 87.2 & 79.5 \\
BGE-large & 81.2 & 78.8 & 65.6 & 91.9 & 85.6 \\
BGE-v2-m3 & 85.1 & 80.4 & 64.6 & 92.9 & 87.1 \\
\bottomrule
\end{tabular}}
\end{table}

\FloatBarrier
\subsection{Comparison with Maximal Marginal Relevance}
\label{app:mmr-comparison}

We compare \method with maximal marginal relevance (MMR)~\citep{carbonell1998mmr}, which balances query relevance against the maximum similarity to an already selected item. While MMR explicitly penalizes pairwise redundancy, \method optimizes a joint query-reconstruction objective over the selected set. Residual scoring and joint refitting work together to select memories that improve this reconstruction. This comparison tests whether reconstruction-driven composition improves on explicit redundancy penalization.

We use the same LoCoMo memory store and per-question candidate pool of $P=200$ memories, with GPT-4o-mini as both answerer and judge, and vary only the memory selection strategy. MMR uses a fixed relevance--diversity coefficient $\lambda_{\mathrm{MMR}}=0.5$. Each budget $K\in\{10,20,50,200\}$ is evaluated on 1,540 questions.

\begin{table}[!htbp]
\centering
\caption{LoCoMo accuracy (\%) for \method and MMR with $\lambda_{\mathrm{MMR}}=0.5$, using the same $P=200$ candidate pool and downstream evaluation. Each row contains 1,540 questions; bold denotes the higher accuracy.}
\label{tab:mmr-comparison}
\small
\setlength{\tabcolsep}{18pt}
\begin{tabular}{rrr}
\toprule
Memory budget $K$ & \method & MMR \\
\midrule
10  & \textbf{81.04} & 77.08 \\
20  & \textbf{85.06} & 83.31 \\
50  & \textbf{87.08} & 86.10 \\
\bottomrule
\end{tabular}
\end{table}

\method achieves higher accuracy at all three tested budgets (Table~\ref{tab:mmr-comparison}). The largest advantage occurs at Top-10, where accuracy increases from 77.08\% to 81.04\%, a gain of 3.96 percentage points; at Top-20, the gain is 1.75 points. These results show the value of residual-guided composition over MMR at the tested coefficient, particularly when only a small fraction of the candidate memories can enter the answer context.

\FloatBarrier
\subsection{Robustness to embedding choice}
\label{sec:embedding-robustness}

We compare five encoders on LoCoMo at Top-10, Top-20, and Top-50, keeping each question's candidate memories fixed and changing only the embeddings used for composition. The encoders are text-embedding-3-small, E5-large-v2, BGE-M3, Qwen3-Embedding-0.6B, and Qwen3-Embedding-8B. Each configuration evaluates 1,540 questions under the same downstream evaluation setup.

Table~\ref{tab:embedding-robustness} reports the results. All tested encoders outperform the fixed hybrid baseline at these budgets. E5 and BGE remain close to the default encoder, while Qwen3-Embedding-8B further improves accuracy at Top-10 and Top-20. These results demonstrate the applicability of \method across representation families and its ability to benefit from stronger embeddings under tight budgets.

\begin{table}[!htbp]
\caption{Robustness to embedding choice on LoCoMo with identical per-question candidate memories. Only the dense encoder used by \method changes. Entries report accuracy (\%) over 1,540 questions; parentheses give changes in percentage points relative to text-embedding-3-small. Bold denotes the highest accuracy at each cutoff.}
\label{tab:embedding-robustness}
\centering
\small
\setlength{\tabcolsep}{5pt}
\renewcommand{\arraystretch}{1.2}
\newcommand{\embacc}[2]{#1\,{\scriptsize($#2$)}}
\newcommand{\embref}[1]{#1\phantom{\,{\scriptsize($+0.00$)}}}
\begin{tabular}{@{}lrrr@{}}
\toprule
Embedding model & Top-10 & Top-20 & Top-50 \\
\midrule
text-embedding-3-small & \embref{81.04} & \embref{85.06} & \embref{\textbf{87.08}} \\
E5-large-v2 & \embacc{80.32}{-0.71} & \embacc{83.57}{-1.49} & \embacc{86.49}{-0.58} \\
BGE-M3 & \embacc{80.00}{-1.04} & \embacc{84.68}{-0.39} & \embacc{86.10}{-0.97} \\
Qwen3-Embedding-0.6B & \embacc{79.94}{-1.10} & \embacc{82.40}{-2.66} & \embacc{85.65}{-1.43} \\
Qwen3-Embedding-8B & \embacc{\textbf{83.18}}{+2.14} & \embacc{\textbf{85.91}}{+0.84} & \embacc{86.49}{-0.58} \\
\bottomrule
\end{tabular}
\end{table}

\FloatBarrier
\subsection{Read-Time Efficiency}
\label{app:efficiency}

\paragraph{Hardware and execution.}
The experimental server has two Intel Xeon Gold 5117 CPUs at a base frequency of 2.00\,GHz, with 28 physical cores and 56 hardware threads in total, and 256\,GiB of RAM, running Ubuntu 20.04.5 LTS. \method runs on the CPU with a Numba backend and a single worker, with BLAS and OpenMP thread counts set to one. Each neural reranking run uses a single NVIDIA GeForce RTX 4090 GPU with 24\,GB of device memory. Qwen3.5 rerankers use BF16 precision.

\paragraph{Neural rerankers and latency measurement.}
We retain the memory store, candidate pool, answerer, and judge from Section~\ref{sec:mechanism-results}, fixing $P=200$ and $K=20$. BGE-base, BGE-large, and BGE-v2-m3 are cross-encoder rerankers; Qwen3.5-0.8B, -2B, and -4B are pointwise LLM scorers. Figure~\ref{fig:reranker-accuracy-latency} uses batch size 128, a large candidate batch size supported by all evaluated models. Each batch contains candidate memories for a single question. We measure final reranking latency, excluding memory construction, candidate retrieval, and answer generation. The figure reports medians. Speedups based on both mean and median latency, together with tail-latency statistics, are reported below. Full-batch BGE inference scores all 200 candidates for a question in a single batch.

\begin{table}[h]
\caption{Final reranking latency at Top-20 with $P=200$ candidates. Neural scorers use batch size 128, a large candidate batch size supported by all evaluated BGE and Qwen models. Tail latency is P95 for every method. The two speedup columns divide each baseline's mean or median latency by the corresponding \method value (17.69 or 8.67\,ms).}
\label{tab:read-latency}
\centering
\small
\setlength{\tabcolsep}{4pt}
\begin{tabular}{lrrrrrr}
\toprule
& & \multicolumn{3}{c}{Latency (ms)} & \multicolumn{2}{c}{Speedup} \\
\cmidrule(lr){3-5}\cmidrule(lr){6-7}
Read policy & Batch & Mean & Median & P95 & Mean latency & Median latency \\
\midrule
\method & -- & 17.69 & 8.67 & 57.76 & 1.00$\times$ & 1.00$\times$ \\
BGE-base & 128 & 109.26 & 109.03 & 126.12 & 6.18$\times$ & 12.58$\times$ \\
BGE-large & 128 & 283.73 & 283.27 & 330.01 & 16.04$\times$ & 32.67$\times$ \\
BGE-v2-m3 & 128 & 275.66 & 274.33 & 320.70 & 15.58$\times$ & 31.64$\times$ \\
Qwen3.5-0.8B & 128 & 1682.95 & 1675.98 & 1751.40 & 95.14$\times$ & 193.31$\times$ \\
Qwen3.5-2B & 128 & 2119.66 & 2113.75 & 2205.36 & 119.82$\times$ & 243.80$\times$ \\
Qwen3.5-4B & 128 & 5151.40 & 5138.73 & 5343.48 & 291.20$\times$ & 592.70$\times$ \\
\bottomrule
\end{tabular}
\end{table}

\method takes 27.24\,s in total over the reported evaluation run. Table~\ref{tab:read-latency} uses the same batch size 128 as Figure~\ref{fig:reranker-accuracy-latency}; each batch contains candidate memories for a single question. Figure~\ref{fig:reranker-batch-latency} reports the batch-size sweeps for both cross-encoder and LLM rerankers, with \method shown at its median selection latency of 8.67\,ms.

\textbf{Accuracy gain relative to hybrid retrieval.}
We measure the fraction of a reranker's accuracy gain recovered by \method relative to the same hybrid-retrieval baseline. Using the displayed accuracies, this fraction is $(85.1-79.6)/(87.1-79.6)\approx 73\%$ for BGE-v2-m3. At batch size 128, the corresponding ratio of median reranking latencies is $8.67/274.329\approx 3.2\%$. The gain fraction compares improvements in answer accuracy; the latency fraction compares median final-reranking time per question.

\textbf{Full-batch comparison.}
For BGE-v2-m3, full-batch inference scores all $P=200$ candidates for a question in a single batch. Its median reranking latency is 285.612\,ms, compared with 8.67\,ms for \method, yielding the $285.612/8.67\approx 32.9\times$ speedup in median latency reported in Section~\ref{sec:efficiency-results} and the abstract. Table~\ref{tab:read-latency-fullbatch} extends the full-batch comparison to three BGE and two Qwen3.5 scorers, reporting mean, median, and P95 latency together with speedups computed from median latency. These results complement the batch-128 comparison in Figure~\ref{fig:reranker-accuracy-latency} and Table~\ref{tab:read-latency}.

\begin{table}[!htbp]
\caption{Full-batch reranking latency at Top-20 with $P=200$ candidates over 1,540 questions. Neural scorers use batch size 200. Latencies are in milliseconds. ``Median speedup'' divides each baseline median by the \method median of 8.67\,ms.}
\label{tab:read-latency-fullbatch}
\centering
\small
\begin{tabular}{lrrrr}
\toprule
Read policy & Mean ms & Median ms & P95 ms & Median speedup \\
\midrule
\method & 17.69 & 8.67 & 57.76 & 1.00$\times$ \\
BGE-base & 111.34 & 110.48 & 129.88 & 12.74$\times$ \\
BGE-large & 286.10 & 286.02 & 331.60 & 32.99$\times$ \\
BGE-v2-m3 & 286.45 & 285.61 & 332.61 & 32.94$\times$ \\
Qwen3.5-0.8B & 1733.86 & 1724.74 & 1794.96 & 198.93$\times$ \\
Qwen3.5-2B & 2129.20 & 2124.70 & 2222.55 & 245.06$\times$ \\
\bottomrule
\end{tabular}
\end{table}

\begin{figure}[htbp]
\centering
\includegraphics[width=0.6\linewidth]{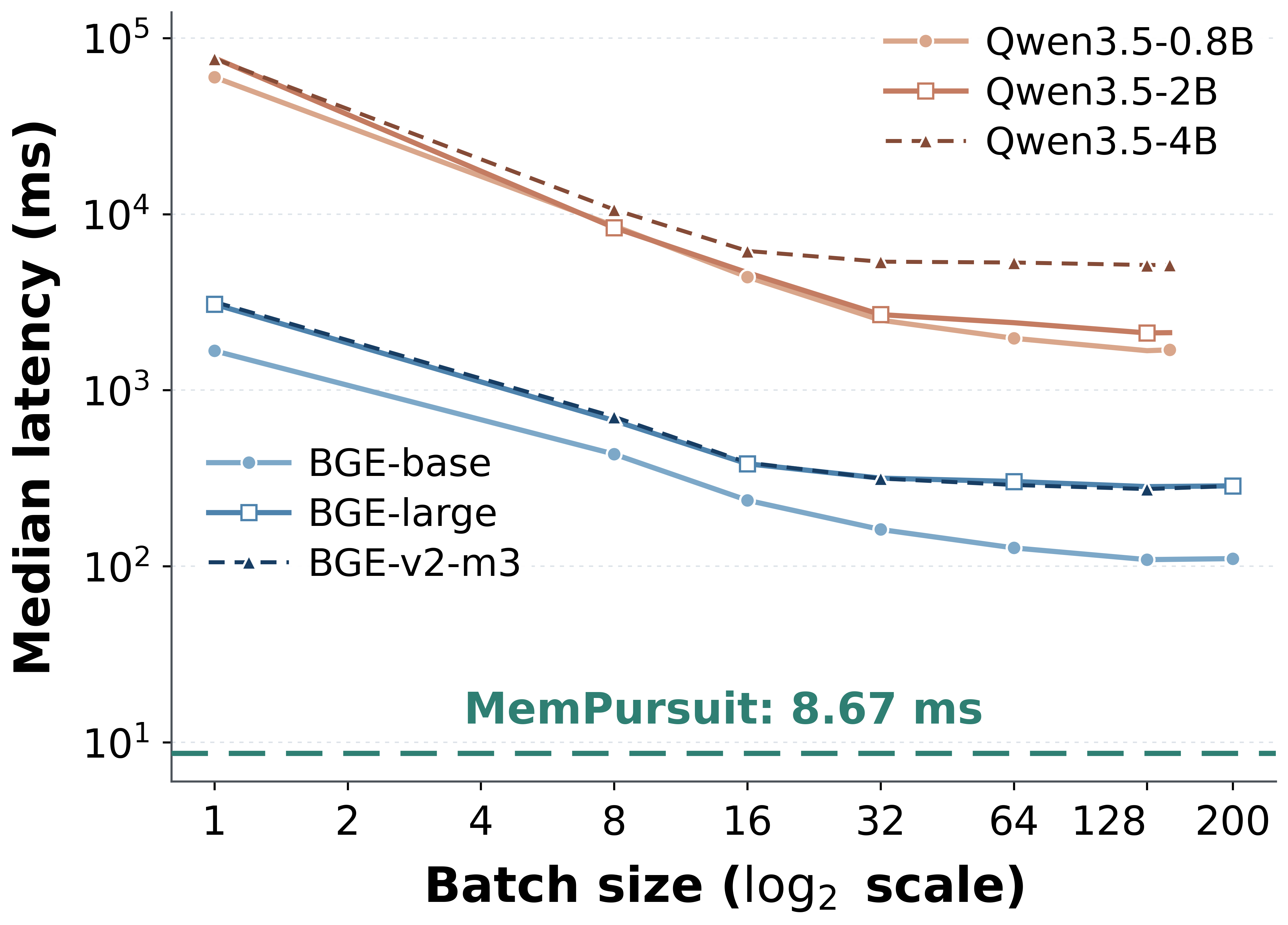}
\caption{Median reranking latency across batch sizes for BGE cross-encoders (blue) and Qwen3.5 LLM scorers (orange), with \method as a batch-independent reference at 8.67\,ms (green dashed line). The horizontal axis uses a base-2 logarithmic scale and the vertical axis a base-10 logarithmic scale. BGE sweeps extend to batch size 200 and Qwen sweeps to 144. }
\label{fig:reranker-batch-latency}
\end{figure}

\FloatBarrier
\subsubsection{Latency across candidate and context budgets}
\label{app:pk-latency}

Table~\ref{tab:pk-latency} reports median reranking latency as we vary the candidate count $P\in\{100,200,300,500\}$ and the selected-memory budget $K\in\{10,20,50,100\}$. Latency increases with both the number of candidates and the number of selected memories, while remaining below 40\,ms throughout the tested grid. At $K=20$, increasing $P$ from 100 to 500 raises the median from 6.30\,ms to 22.47\,ms. Even the largest tested configuration, $P=500$ and $K=100$, requires only 38.58\,ms, demonstrating inexpensive selection across a broad range of candidate and context budgets.

\begin{table}[!htbp]
\caption{Median \method reranking latency (ms per question) across candidate counts $P$ and selected-memory budgets $K$. Each entry is measured in the $P/K$ sweep.}
\label{tab:pk-latency}
\centering
\small
\setlength{\tabcolsep}{12pt}
\begin{tabular}{rrrrr}
\toprule
& \multicolumn{4}{c}{Selected-memory budget $K$} \\
\cmidrule(lr){2-5}
Candidates $P$ & 10 & 20 & 50 & 100 \\
\midrule
100 & 4.54 & 6.30 & 11.79 & 22.75 \\
200 & 7.27 & 8.40 & 13.62 & 23.99 \\
300 & 9.73 & 12.16 & 17.11 & 26.49 \\
500 & 18.04 & 22.47 & 30.67 & 38.58 \\
\bottomrule
\end{tabular}
\end{table}

\end{document}

%% file: math_commands.tex
\usepackage{amsmath,amsfonts,bm}

\def\eqref#1{equation~\ref{#1}}

\def\1{\bm{1}}

\DeclareMathAlphabet{\mathsfit}{\encodingdefault}{\sfdefault}{m}{sl}
\SetMathAlphabet{\mathsfit}{bold}{\encodingdefault}{\sfdefault}{bx}{n}

